\documentclass[11pt]{jmlr}
\jmlrproceedings{Preprint}{Preprint}
\jmlrvolume{}
\jmlrpages{}
\jmlryear{2026}
\jmlrworkshop{}
\editor{}
\usepackage{times}
\usepackage{booktabs}
\usepackage{array}
\usepackage{mathtools}
\usepackage{microtype}
\usepackage{needspace}
\usepackage[all]{hypcap}
\usepackage{enumitem}
\setlist[enumerate]{itemsep=3pt,topsep=4pt}
\makeatletter
\patchcmd{\@begintheorem}{\hskip\labelsep}
  {\MakeLinkTarget{\@currentcounter}\hskip\labelsep}
  {}{\PackageError{manuscript-links}{Cannot patch theorem anchor}{}}
\patchcmd{\NRorg@opargbegintheorem}{\hskip\labelsep}
  {\MakeLinkTarget{\@currentcounter}\hskip\labelsep}
  {}{\PackageError{manuscript-links}{Cannot patch named theorem anchor}{}}
\makeatother
\BeforeBeginEnvironment{theorem}{\Needspace{5\baselineskip}}
\BeforeBeginEnvironment{lemma}{\Needspace{4\baselineskip}}
\BeforeBeginEnvironment{proposition}{\Needspace{4\baselineskip}}
\BeforeBeginEnvironment{corollary}{\Needspace{4\baselineskip}}
\newcommand{\cX}{\mathcal X}
\newcommand{\cY}{\mathcal Y}
\newcommand{\cH}{\mathcal H}
\newcommand{\cD}{\mathcal D}
\newcommand{\cG}{\mathcal G}

\newcommand{\BDS}{\operatorname{BDS}}
\newcommand{\aBDS}{\operatorname{aBDS}}
\newcommand{\DS}{\operatorname{DS}}
\newcommand{\wid}{\operatorname{R}}
\newcommand{\mB}{\mathrm{m}^{\mathrm B}}
\newcommand{\err}{\operatorname{err}}
\newcommand{\F}{\mathbb F}
\newcommand{\eps}{\varepsilon}
\newcommand{\ind}[1]{\mathbf 1\{#1\}}
\newcommand{\E}{\mathbb E}
\newcommand{\AM}{\cH^{\mathrm{AM}}_n}
\newcommand{\proofref}[1]{\textit{Proof: Appendix~\ref{#1}.}}
\newcommand{\role}[1]{\par\noindent\textit{Role in the argument.} #1\par}
\providecommand{\doi}{}
\renewcommand{\doi}[1]{doi: \href{https://doi.org/#1}{\nolinkurl{#1}}}
\hypersetup{pdftitle={Bandit Multiclass PAC Learning: Structural Baselines and a Confidence Direct Sum},pdfauthor={Guangjian Zhang, Zhengyi Yang, Xiaoyang Wang},pdfsubject={Preprint},colorlinks=true,linkcolor=blue,citecolor=blue,urlcolor=blue}
\title[Bandit Multiclass PAC Learning]{Bandit Multiclass PAC Learning:\\Structural Baselines and a Confidence Direct Sum}
\author[Zhang, Yang and Wang]{
\Name{Guangjian Zhang} \Email{\href{mailto:guangjian.zhang@student.unsw.edu.au}{guangjian.zhang@student.unsw.edu.au}}\\
\addr UNSW Sydney
\AND
\Name{Zhengyi Yang} \Email{\href{mailto:zhengyi.yang@sydney.edu.au}{zhengyi.yang@sydney.edu.au}}\\
\addr The University of Sydney
\AND
\Name{Xiaoyang Wang} \Email{\href{mailto:xiaoyang.wang1@unsw.edu.au}{xiaoyang.wang1@unsw.edu.au}}\\
\addr UNSW Sydney
}
\begin{document}
\maketitle
\begin{abstract}%
How much does high confidence cost when a multiclass learner sees only whether its prediction is correct? We separate two behaviors in realizable bandit PAC learning. An anchored dimension gives a general lower bound that is attained, up to universal constants, by a family of product-block classes. In contrast, for $n\ge16$ divisible by four, an affine multiplexer with bandit DS dimension $2n-1$ has sample complexity
$\Theta((n\min\{n,\ln(1/\delta)\}+\ln(1/\delta))/\eps)$
in the standard accuracy and confidence regime. Its lower bound applies to adaptive randomized learners with improper outputs. The obstruction is that evidence collected against one candidate group cannot be reused against the others. The better of two algorithms matches this cost. One finds the correct group through a positive response. The other exhausts the rank of successive groups. The resulting confidence profile rules out a universal additive-confidence bound linear in the bandit DS dimension. Moreover, classes with identical anchored dimension, bandit DS dimension, and pointwise label width can have polynomially different sample complexities. We also sharpen the confidence accounting in ListCascade and show how one positive response replaces the bandit DS dimension by the anchored dimension in a general upper bound.
\end{abstract}
\section{Introduction}\label{sec:intro}
In multiclass learning with bandit feedback, a learner proposes a label and observes only whether it is correct. A negative response eliminates the proposed label but does not reveal the right one. This protocol captures a basic information constraint: evaluating a proposed classification can be easier than supplying a label from a large alphabet. The statistical question is how many such evaluations suffice to learn a predictor that is accurate with high probability.

\citet{DanielyEtAl2011} introduced this question in the PAC framework. A recent development is the bandit DS dimension of \citet{HMMSv2}, which measures the total coordinatewise branching in a finite restriction of the hypothesis class. Their ListCascade algorithm uses successively smaller prediction lists to obtain a general realizable upper bound. This connects bandit PAC learning to the combinatorial theory of multiclass and list learnability. The revised version retains a finer open question \citep[Section~5]{HMMSv3}: does confidence cost only an additive $\ln(1/\delta)/\eps$ term in a bound linear in this dimension, or can the structure of the class multiply that cost?

We show that both behaviors occur, even for classes with the same natural dimension parameters. Here $\eps$ is the target error and $\delta$ is the allowed failure probability. Write $B$ for the bandit DS dimension and $R$ for the maximum number of labels realized at a single instance. We introduce an anchored variant $A$ of $B$, defined in Section~\ref{sec:model}. It measures branching that can be placed on rare instances while retaining a separate instance to carry the remaining probability mass. The parameters satisfy
\begin{equation}\label{eq:chain-intro}
\max\{A,R-1\}\le B\le A+R-1.
\end{equation}
The distinction between $A$ and $B$ matters even at constant confidence. The further distinction between classes with the same $(A,B,R)$ emerges when confidence increases.

\subsection{Main Findings and the Route to Them}
Our results follow a baseline--attainment--obstruction progression. Each step answers a question left by the previous one.

\textit{First separate label search from rare-instance learning.}
A label that distinguishes hypotheses at every instance can be tested immediately. Its elimination need not incur a factor $1/\eps$. By contrast, learning a label that varies only on a set of probability of order $\eps$ requires waiting for that set. The anchored dimension separates these costs. For every class, in the range $0<\eps\le1/16$ and $0<\delta\le1/4$, we prove
\begin{equation}\label{eq:baseline-intro}
\mB_\cH(\eps,\delta)\ge
\frac1{96}\left(R-1+\frac{A+\ind{A>0}\ln(1/\delta)}{\eps}\right).
\end{equation}
Here $\mB_\cH$ denotes the optimal number of bandit training rounds. When $A=0$, we determine this number exactly for all $\eps,\delta\in(0,1)$. For a larger family of private-shell anchored-block classes, abbreviated PAB, an explicit algorithm attains \eqref{eq:baseline-intro} up to universal constants. It tests globally identifiable hypotheses, locates the active block after a negative response, and updates only the point at which a mistake occurs.

\textit{Next determine when confidence must be paid for repeatedly.}
The PAB algorithm succeeds because one observation can locate the relevant block. Our affine multiplexer prevents this localization from happening cheaply at high confidence. It contains $n$ candidate groups, each carrying a binary linear class. A negative prediction constrains only the group named in that prediction. We prove the matching rate
\begin{equation}\label{eq:profile-intro}
\mB_{\AM}(\eps,\delta)
=\Theta\left(\frac{n\min\{n,L\}+L}{\eps}\right),
\qquad L:=\ln(1/\delta),
\end{equation}
for $n\ge16$ divisible by four and in the same $(\eps,\delta)$ range. The class has $A=B=2n-1$ and $R=2n$. Up to the rank scale, confidence incurs a cost of order $nL/\eps$. Beyond that scale, an algorithm with a quadratic mistake budget gives the rate $(n^2+L)/\eps$.

This is a confidence direct sum within a single learning task. The term refers to evidence that must be accumulated separately for candidate groups, rather than to independent tasks supplied to the learner. Its proof must account for a learner that sees the current instance before choosing which group to test. We handle this adaptivity by controlling the union of the low-rank row spaces of all groups. A fresh instance outside that union cannot eliminate any low-rank group in one step. Counting the surviving groups then converts this obstruction into a PAC lower bound.

Two consequences make the separation concrete. There is no universal bound of order $(B+L)/\eps$ over all classes. Also, a rare-label class and the multiplexer have identical $(A,B,R)$ but complexities of order $n/\eps$ and $n^2/\eps$ at the same confidence. Thus these parameters alone cannot describe the confidence profile to within fixed polylogarithmic factors in $A,R,1/\eps,L$.

\textit{Finally connect the separation to general algorithms.}
We combine a sharper full-information list input from \citet{Pabbaraju26} with a refined allocation of stage errors and failure probabilities in ListCascade to obtain
\begin{equation}\label{eq:upper-intro}
\mB_\cH(\eps,\delta)
=O\left(\frac{B\lambda^2+R(1+L)}{\eps}\right),
\qquad \lambda:=\max\{1,\log_2R\}.
\end{equation}
A second module first waits for a positive response and then restricts to hypotheses consistent with that response. The surviving class has bandit DS dimension at most $A$. At constant confidence this yields
$O(R+(A+1)\log^2(A+2)/\eps)$.
The algorithmic picture is therefore coherent: localization gives additive costs on PAB classes, while the multiplexer shows why a general learner can be forced to pay for confidence across many groups.

\subsection{Relation to Earlier Work}
The full-information theory provides the combinatorial starting point. \citet{DanielyShalev14} introduced the DS dimension; \citet{BrukhimEtAl22} proved that its finiteness characterizes multiclass PAC learnability. \citet{CharikarP23} developed the corresponding theory for list predictions. \citet{Pabbaraju26} subsequently bounded list one-inclusion density by the corresponding DS dimension. Combined with the suffix-risk conversion of \citet{AdenAliEtAl23}, this supplies the stronger list input used here. We retain ListCascade's reduction through shrinking lists, refine its resource allocation, and introduce the positive-response restriction.

Bandit classification also has an online and agnostic literature. \citet{DanielyHelbertal13} study the online mistake-bound price of bandit information. \citet{Erez24} obtain refined finite-class regret bounds, and \citet{ErezFast24} establish fast agnostic PAC rates. The universal-learning analysis of \citet{HSZ2025} studies rates for fixed distributions and includes relevant trial-and-guess examples. \citet{HSW2025} characterize non-uniform bandit learnability, where the required sample size may depend on the target hypothesis. Here the guarantee is uniform over the class and instance distribution, and we study its dependence on confidence. Our focus is on matching structured-family rates and the confidence separation.

\textit{Version-specific comparison.}
Theorem~4.2 of \citet{HMMSv2} states a lower bound of $\Omega((B+\ln(1/\delta))/\eps)$ with an unspecified nontriviality condition. Section~\ref{sec:baseline} explains why agreement of two hypotheses at an anchor does not suffice. The subsequent v3 revision \citep{HMMSv3} instead requires a dimension-attaining witness on at least two instances, supplies conditional equality-query counting, and explicitly treats effective label width. We distinguish these revisions from our all-class anchored bound. In \eqref{eq:upper-intro}, the stronger external list input reduces the dimension term, while our allocation sharpens the confidence accounting. Appendix~\ref{app:versions} records the precise comparison.

\textit{Reading route.}
Section~\ref{sec:model} defines the protocol and dimensions. Sections~\ref{sec:baseline} and~\ref{sec:pab} establish and attain the structural baseline. Section~\ref{sec:am} contains the confidence separation. Section~\ref{sec:upper} returns to general upper bounds. Complete proofs follow the same order in the appendices, with a result-to-proof index and two principal dependency diagrams in Appendix~\ref{app:roadmap}.

\section{Protocol and Structural Parameters}\label{sec:model}
Let $\cX$ be a nonempty instance space, $\cY$ a finite label set of size $K$, and $\emptyset\ne\cH\subseteq\cY^{\cX}$ a known class. A distribution $\cD$ on $\cX$ and a target $h^\star\in\cH$ are fixed before learning. In round $t$, the learner observes an independent draw $x_t\sim\cD$, predicts $\hat y_t\in\cY$, and receives the bit
$b_t=\ind{\hat y_t=h^\star(x_t)}$.
Its prediction may depend on the entire history, the current instance, and internal randomness. After $m$ rounds it outputs a function $\hat h:\cX\to\cY$, not necessarily in $\cH$. Randomness in the output is included in the learner's seed. Its risk is
$\err_\cD(\hat h)=\Pr_{x\sim\cD}[\hat h(x)\ne h^\star(x)]$.
The sample complexity $\mB_\cH(\eps,\delta)$ is the smallest $m\in\mathbb N_0$ for which one learner satisfies
$\Pr[\err_\cD(\hat h)\le\eps]\ge1-\delta$
for every such $(\cD,h^\star)$; it is infinite if no such $m$ exists.

We study information-theoretic sample complexity, allowing unrestricted computation on the known class. Prediction rules, selections from label sets, and events below are assumed measurable. Write $[d]=\{1,\ldots,d\}$, $\mathbb N=\{1,2,\ldots\}$, $\mathbb N_0=\{0,1,\ldots\}$, and use $\ln$ for the natural logarithm. Constants in $O,\Omega,\Theta$ are universal unless specified otherwise. A risk guarantee concerns the final classifier, not the number of errors during training.

\begin{definition}[Pseudo-box dimensions]\label{def:dimensions}
For a finite sequence $S=(x_1,\ldots,x_d)$, write
$\cH|_S=\{(h(x_1),\ldots,h(x_d)):h\in\cH\}$.
Two vectors are $i$-neighbors if they differ in coordinate $i$ and agree in every other coordinate. A witness with multiplicities $(N_1,\ldots,N_d)$ is a finite nonempty set $V\subseteq\cH|_S$ such that, for every $v\in V$ and every coordinate $i\in[d]$ with $N_i>0$, there are at least $N_i$ distinct $i$-neighbors of $v$ in $V$.

An \emph{ordinary witness} has $d\ge1$ and all $N_i\in\mathbb N$. The bandit DS dimension $B=\BDS(\cH)$ is the supremum of $\sum_iN_i$ over ordinary witnesses \citep[Definitions~A.2--A.3]{HMMSv3}. An \emph{anchored witness} has $d\ge2$, all $N_i\in\mathbb N_0$, and at least one coordinate with $N_i=0$, called an anchor. The anchored dimension $A=\aBDS(\cH)$ is the supremum of $\sum_iN_i$ over anchored witnesses. Either supremum is zero if there is no witness of positive value. Finally,
$R=\wid(\cH):=\sup_{x\in\cX}|\{h(x):h\in\cH\}|$.
\end{definition}

An anchor has no required neighbors, but positive-coordinate neighbors must still agree there. Thus its label need not be common to the whole class, or even to the whole witness. This local agreement is exactly what the lower-bound argument uses. Proposition~\ref{prop:alternative-anchor} in Appendix~\ref{app:alternative-anchor} compares this definition with zeroing the smallest multiplicity of an ordinary witness. The parameter $R$ counts labels that can actually occur at a point, rather than labels merely present in the ambient alphabet.

The relation \eqref{eq:chain-intro} and the equivalence
$A>0$ if and only if some two hypotheses agree at one point and disagree at another
are proved in Lemma~\ref{lem:dimensions}. We use these facts throughout. Except where a result gives a wider range explicitly, the lower and matching bounds use $0<\eps\le1/16$ and $0<\delta\le1/4$. We retain the abbreviation $L=\ln(1/\delta)$.

\section{A Structural Baseline for Every Class}\label{sec:baseline}
The first question is which structural costs really require rare observations. Two small classes expose the distinction. Let $\cX=\{a,b\}$ and list each hypothesis by its values at $(a,b)$. The constant class consists of $(j,j)$ for $j\in[K]$. One label test eliminates a hypothesis globally, so its sample complexity has no accuracy dependence. The class
\[
\cH_{\mathrm{anc}}=\{(1,1),(1,2)\}\cup\{(j,j):3\le j\le K\},\qquad K\ge3,
\]
adds a pair that agrees at $a$ and differs at $b$. Test the private labels $3,\ldots,K$ first. If all fail, only that pair remains, and an occurrence of $b$ distinguishes it. Consequently,
$\mB_{\cH_{\mathrm{anc}}}(\eps,\delta)\le K-2+\lceil L/\eps\rceil$.
Both classes have $B=K-1$, but their anchored dimensions are $0$ and $1$, respectively; Appendix~\ref{app:benchmarks} gives the proofs. At $\eps=1/K$ and fixed confidence, the second example needs only $O(K)$ samples. Agreement of a pair at an anchor therefore does not support a universal lower bound of order $B/\eps$.

\subsection{The Three Costs and Their Joint Guarantee}
The baseline separates three mechanisms. The width $R$ forces label search even on a point mass. An anchored witness of value $A$ places its branching on rare points and forces a cost $A/\eps$. When such branching exists, two hypotheses agreeing at the anchor also force a confidence cost $L/\eps$.

\begin{theorem}[All-class baseline]\label{thm:threepart}
For every nonempty class and $0<\eps\le1/16$, $0<\delta\le1/4$,
\[
\mB_\cH(\eps,\delta)\ge
\frac1{96}\left(R-1+\frac{A+\ind{A>0}L}{\eps}\right).
\]
If $A=\infty$, no finite sample size suffices. The lower bound allows adaptive randomized learners and improper outputs.
\end{theorem}
\proofref{app:threepart}

The main issue is to make the rare-point argument valid for an arbitrary learner. A pseudo-box need not admit a symmetry that justifies uniform guessing. Instead, reveal all target coordinates except one for free. Conditional on these revealed coordinates, the unknown coordinate is uniform on a fiber containing at least one more label than its multiplicity. After at most $c$ equality queries, a learner can identify at most $c+1$ values, including its final guess. Lemma~\ref{lem:equality} proves this counting statement with a pathwise query cap; Theorem~\ref{thm:anchored-lower} applies it after conditioning on the instance sequence. The width and confidence arguments are separate, and are combined only at the end.

\subsection{The Zero-Anchored-Dimension Endpoint}
If $A=0$, any two distinct hypotheses disagree everywhere. A positive response therefore identifies the target outright, and a negative response eliminates the tested hypothesis. Randomize their testing order and, after the prescribed rounds, guess among the survivors.

\begin{theorem}[Exact endpoint]\label{thm:zero}
If $A=0$, then $|\cH|=R$, and for all $\eps,\delta\in(0,1)$,
\[
\mB_\cH(\eps,\delta)=\max\{0,\lceil(1-\delta)R\rceil-1\}.
\]
\end{theorem}
\proofref{app:zero}

This endpoint identifies a genuine label-search cost rather than a rare-event cost. The next section adds rare points in a controlled way and shows that the full baseline can still be attained.

\section{When the Baseline Is Attained}\label{sec:pab}
We now construct classes where one negative response locates all remaining uncertainty. Fix a base function $h_0$ and pairwise disjoint finite blocks $(X_j)_{j\in J}$. At a point $x\in X_j$, let $A_{j,x}$ be a finite label set containing $h_0(x)$, and let $w_{j,x}=|A_{j,x}|-1$. The block class $\cH_j$ consists exactly of the assignments from $\prod_{x\in X_j}A_{j,x}$, extended by $h_0$ outside the block. Each block with positive total width has an external point $a_j\notin X_j$.

The core is explicitly $\cH_{\mathrm{core}}=\{h_0\}\cup\bigcup_{j\in J}\cH_j$, including when $J$ is empty. Add $q$ private hypotheses $g_1,\ldots,g_q$: at every point, their labels are distinct from one another and from every core label. Define the resulting PAB class by $\cH=\cH_{\mathrm{core}}\cup\{g_1,\ldots,g_q\}$. Set
\[
M_j=\sum_{x\in X_j}w_{j,x},\qquad M=\sup_{j\in J}M_j,
\qquad w_{\max}=\sup_{j,x}w_{j,x},
\]
with empty suprema zero. The following modules receive this representation and the requested $(\eps,\delta)$.

\subsection{Test, Locate, and Update Conservatively}\label{sec:pab-actions}
\textit{Test private hypotheses.}
On round $s\le q$, provided all earlier private tests were negative, predict $g_s(x_s)$. A positive response identifies the target as $g_s$. In that case, output $g_s$ and do not enter the core phase; any remaining allotted rounds may be ignored. A negative response eliminates $g_s$ globally. Enter the core phase only after all $q$ tests return negative.

\textit{Locate the active block.}
Start the core phase with the predictor $h_0$. Its budget includes both the wait for localization and all subsequent updates. A first negative response at a point $x$ identifies the unique block containing $x$ as the active one, since a core target differs from $h_0$ in at most one block. Immediately apply the pointwise update below at $x$. If no negative occurs during the allotted phase, retain $h_0$.

\textit{Update only the mistaken point.}
At each point $x$ of the active block $X_j$, keep a fixed ordering of $A_{j,x}$, with $h_0(x)$ first. Remove the current candidate only after a negative response at that same point and move to the next candidate. Do not change predictions at other points. A correct current label is never removed. Output the current predictor after
$\lceil\max\{2M,8L\}/\eps\rceil$ total core-phase rounds when $M>0$, counting the locating round and its update; if $M=0$ and all $q$ private tests were negative, output $h_0$ immediately.

\begin{theorem}[PAB guarantee]\label{thm:pab}
For every PAB class with $M<\infty$,
\[
A=M,\qquad B=\max\{M,R-1\},\qquad R=q+1+w_{\max}.
\]
For every $\eps,\delta\in(0,1)$, the modules above use at most
$q+\ind{M>0}\lceil\max\{2M,8L\}/\eps\rceil$ rounds and are $(\eps,\delta)$-PAC. In the standard range,
\[
\mB_\cH(\eps,\delta)=\Theta\left(R-1+\frac{M+\ind{M>0}L}{\eps}\right).
\]
\end{theorem}
\proofref{app:pab}

Updating only the mistaken point preserves every currently correct prediction, so the error set only shrinks. Lemma~\ref{lem:pab-algorithm} verifies this property and bounds the number of negative responses after private elimination by $M$. Lemma~\ref{lem:monotone} converts these two guarantees into a population-risk bound for the final predictor.

The family includes the two examples of Section~\ref{sec:baseline}. It also includes the rare-label class
$\cG_K=\{g_y:y\in[K]\}$ on $\{a,b\}$ with $g_y(a)=1$ and $g_y(b)=y$.
Here $A=B=K-1$, $R=K$, and the rate is $\Theta((K+L)/\eps)$ for $K\ge2$. Disjoint blocks can also carry different label scales; Appendix~\ref{app:scales} shows that the same localization rule handles them without paying separately for every scale. The next construction removes this shortcut: uncertainty is no longer confined to a block that one negative response identifies.

\section{A Confidence Cost That Cannot Be Shared}\label{sec:am}
The PAB learner benefits from a decisive negative response. To force a different behavior, we give the learner many candidate groups and make each negative response informative about only the group it names. The construction is finite, and all the lower bounds in this section cover randomized learners and improper outputs.

\subsection{Constructing the Groups}\label{sec:am-construction}
Let $\F_2=\{0,1\}$ be the binary field. For an integer $n\ge1$, put $V_n=\F_2^n\setminus\{0\}$ and take the instance space $\{a\}\cup V_n$. The label set is $\{\bot\}\cup([n]\times\F_2)$. For $j\in[n]$ and $\theta\in\F_2^n$, define
\[
h_{j,\theta}(a)=\bot,\qquad
h_{j,\theta}(x)=(j,\langle\theta,x\rangle)\quad(x\in V_n),
\qquad \AM=\{h_{j,\theta}:j\in[n],\theta\in\F_2^n\}.
\]
The inner product is over $\F_2$. The index $j$ identifies a group and $\theta$ determines its within-group labeling. The anchor $a$ has the same label for every target.

The construction has $A=B=2n-1$ and $R=2n$, as proved in Lemma~\ref{lem:am-dimensions}. Thus the baseline of Theorem~\ref{thm:threepart} is only of order $(n+L)/\eps$. We now show why this baseline ceases to be sufficient as confidence increases.

\subsection{Lower-Bound Module: Retain Many Plausible Groups}\label{sec:am-lower}
Place mass $1-4\eps$ at $a$ and spread the remaining mass uniformly on $V_n$. Call the non-anchor observations informative rounds. A predictor of risk at most $\eps$ must output the true group on more than half of $V_n$. Learning therefore requires identifying the group, even when the output is improper.

For the analysis, fix the learner's seed and follow the record in which every informative response is negative. The anchor still gives its actual response. A prediction $(g,b)$ at $x\in V_n$ imposes the equation $\langle\theta,x\rangle=1-b$ on candidate group $g$. It imposes no constraint on another group. We count the targets whose genuine interactions with the learner produce this canonical record.

\begin{theorem}[Confidence direct-sum lower bound]\label{thm:interval}
Let $n\ge16$ be divisible by four, $0<\eps\le1/16$, and
$2^{-n/4-2}\le\delta\le1/4$. Define
$\ell=\min\{n/4,\lfloor\log_2(1/\delta)\rfloor-1\}$.
Then
\[
\mB_{\AM}(\eps,\delta)\ge
\left\lfloor\frac{n\ell}{64\eps}\right\rfloor+1
>\frac{n\ln(1/\delta)}{192\ln2\,\eps}.
\]
\end{theorem}
\proofref{app:interval}

The key mechanism concerns any learner restricted to $m\le n\ell/(64\eps)$ rounds, with $\ell$ as in the theorem. Fix its seed. Immediately before the next informative vector $x$ arrives on the canonical record, let $\mathcal G_{<\ell}$ be the currently consistent groups of rank below $\ell$, and let $W_g$ be group $g$'s coefficient row space. The geometric estimate used in Lemma~\ref{lem:lowrank} is
\[
\Pr\!\left[x\in\bigcup_{g\in\mathcal G_{<\ell}}W_g
\;\middle|\;\text{past canonical history}\right]
\le\frac{n2^{\ell-1}}{2^n-1}.
\]
This union is fixed before the learner sees $x$ and chooses a group. Outside it, a query to any consistent low-rank group adds an independent equation and cannot eliminate that group. Lemma~\ref{lem:lowrank} controls these events over the execution under the stated round budget, not just at one observation.

On the good event of Lemma~\ref{lem:lowrank}, the canonical execution contains at most $n\ell/4$ informative rounds. A group that reaches rank $\ell$ or becomes inconsistent must receive at least $\ell$ informative queries addressed to that group. Each query addresses at most one group, so at least $3n/4$ groups remain consistent with rank below $\ell$.

The remaining step is statistical. For a fixed seed and instance sequence, a final consistent system of rank $r$ retains $2^{n-r}$ parameters. Under the uniform prior on the $n2^n$ targets, these have mass $2^{-r}/n$ for that group. This is prior mass, not a posterior probability. Lemma~\ref{lem:canonical} identifies these as genuine targets producing the same record. Lemma~\ref{lem:survivors} then turns their surviving mass into unavoidable failure probability. Thus high confidence requires evidence addressed separately to many groups.

\subsection{Upper-Bound Modules: Find a Positive or Exhaust Rank}\label{sec:am-algorithms}
The learner knows $n$ and the requested $(\eps,\delta)$. It chooses between two procedures using the deterministic sample budgets below. Both predict $\bot$ at the anchor, whose positive response does not identify a group.

\textit{Positive-search procedure.}
On informative rounds, cycle through the $n$ groups and predict an independent fair bit within the current group. An informative positive identifies the true group. Keep only that identity, discard the search data, and ignore the remaining observations in the reserved search phase. On ignored rounds predict $\bot$ and discard the feedback. After this fixed phase ends, use fresh rounds to learn the group's linear labeling conservatively. Start with an empty affine system: on an instance in its row space, predict the forced bit; otherwise predict zero. A negative response at an independent row supplies an equation with right-hand side one. This update corrects that point and never invalidates a correct prediction. After the last learning round and its update, output the maintained predictor. If search ends without an informative positive, output the fixed hypothesis $h_{1,0}$. The search and learning phase lengths are specified in Appendix~\ref{app:am-upper}.

\textit{Rank-exhaustion procedure.}
Test groups in a fixed order. Within a group use the same affine-system predictor. A negative on an independent row increases rank. A negative on a row whose value is already forced eliminates that group, after which the system is reset for the next group. An informative positive confirms the group and prevents further switching. After the last allotted round and its update, output the maintained predictor. Each wrong group costs at most $n+1$ negatives; the true group costs at most $n$. A wrong group misclassifies all informative points, so switching between wrong groups leaves the error set unchanged. In the true group, every maintained equation is correct and expanding the row space can only remove errors. Lemma~\ref{lem:saturation} verifies this monotonicity and the budget $n^2+n-1$. Together, these properties allow Lemma~\ref{lem:monotone} to bound the population risk of the maintained final predictor.

\begin{proposition}[Two matching algorithmic guarantees]\label{prop:am-upper}
For $n\ge16$, $0<\eps\le1/16$, and $0<\delta\le1/4$, positive search uses at most $10nL/\eps$ rounds. Rank exhaustion uses at most $9(n^2+L)/\eps$ rounds. Both return an $(\eps,\delta)$-PAC predictor for $\AM$ under every instance distribution.
\end{proposition}
\proofref{app:am-upper}

The two procedures exploit different sources of certainty. Search needs only one informative success but pays for the confidence of obtaining it. Rank exhaustion spends a bounded amount of effort on each candidate group, after which further confidence costs only an additive term. Selecting the better budget gives the following profile.

\begin{theorem}[Matching confidence profile]\label{thm:profile}
Let $n\ge16$ be divisible by four, $0<\eps\le1/16$, $0<\delta\le1/4$, and $\Psi_n(L)=n\min\{n,L\}+L$. Then
\[
\frac{\Psi_n(L)}{512\eps}\le
\mB_{\AM}(\eps,\delta)\le\frac{10\Psi_n(L)}{\eps}.
\]
\end{theorem}
\proofref{app:profile}

Within the range of Theorem~\ref{thm:profile}, set $\delta_n=2^{-n/4-2}$. For $\delta\ge\delta_n$, the rate is $\Theta(nL/\eps)$. For $\delta\le\delta_n$, it is $\Theta((n^2+L)/\eps)$. Here $\delta_n$ is the endpoint of Theorem~\ref{thm:interval}, whose lower bound combines with the ordinary two-target confidence bound to cover the second range. Appendix~\ref{app:large-delta} separately treats the zero-sample endpoint near $\delta=1$.

\Needspace{9\baselineskip}
\subsection{What the Separation Rules Out}\label{sec:consequences}
\begin{corollary}[Uniform additive bounds and identical-parameter classes]\label{cor:separation}
No universal constant $C$ satisfies
$\mB_\cH(\eps,\delta)\le C(B+\ln(1/\delta))/\eps$
for all classes in the standard range. More generally, an additive bound
$C(f(\cH)+\ln(1/\delta))/\eps$ is impossible if
$f(\AM)=o(n^2)$ as $n$ tends to infinity through multiples of four.

For every $n\ge16$ divisible by four, the rare-label class $\cG_{2n}$ and the multiplexer $\AM$ have the same triple,
\[
(A,B,R)=(2n-1,2n-1,2n).
\]
At $\eps=1/16$ and $\delta=2^{-n/4-2}$ their complexities are respectively $\Theta(n/\eps)$ and $\Theta(n^2/\eps)$.
A positive-valued function depending only on $A,B,R,\eps,\delta$ therefore has multiplicative error at least $c\sqrt n$ on one of these classes, for a universal $c>0$. Here the error is $\max\{z/m,m/z\}$, where $z$ is the function's value and $m$ is that class's sample complexity.
\end{corollary}
\proofref{app:consequences}

This answers the universal-constant form of the additive-rate question retained in \citet[Section~5]{HMMSv3}. That question ranges over every class of finite bandit DS dimension, unlike their lower-bound theorem's restricted nontriviality assumption. The same-parameter comparison identifies a stronger issue than a missing logarithm in one algorithm: the structural parameters themselves fail to record how confidence evidence can be shared. The general upper bounds below locate the corresponding cost on the algorithmic side.

\section{General Upper Bounds and the Role of Positive Feedback}\label{sec:upper}
The structured algorithms exploit special geometry. For an arbitrary class with $B<\infty$, our starting point is ListCascade. We first specify its resource allocation. We then build a two-phase learner that obtains a positive response before running the cascade on the surviving class. These modules use full labels only when a positive bandit response has actually revealed them.

\subsection{Collect Positive Examples and Shrink the Lists}\label{sec:cascade}
A list predictor $\mu$ assigns a set of labels to each point. Its coverage error is
$\Pr_{x\sim\cD}[h^\star(x)\notin\mu(x)]$.
Start with all realized labels at each point, padding with distinct labels to obtain lists of size exactly $R$. Uniform guessing reveals a correctly labeled example whenever the bandit response is positive. For any fixed list predictor of common width $s$, a covered example is retained with probability exactly $1/s$, whereas an uncovered example is never retained. Thus retention adds no instance-dependent reweighting within the coverage event. When coverage has positive probability, the retained law is the original labeled distribution conditioned on coverage.

Run $T=\lfloor\log_2R\rfloor$ stages when $R\ge2$; if $R=1$, output the common hypothesis without samples. At stage $t$, use a full-information list learner on the accepted examples to obtain lists of size
$s_t=2\lceil R/2^{t+1}\rceil-1$.
The last list has size one and defines the output. Proposition~\ref{prop:list-blackbox} supplies a deterministic list learner by combining the sharper full-information input of \citet{Pabbaraju26} with suffix aggregation. Its proof gives the precise interface and constants. Lemma~\ref{lem:harvest} verifies the joint law of the retained examples, including after conditioning on collecting enough within a fixed raw-round budget. This permits their use as independent labeled inputs to the list learner.

\textit{Allocate error and confidence.}
At stage $t$, let $\eta_t$ be the target coverage error on the retained-example distribution. Give collection and list learning a failure allowance $\alpha_t$ each. We choose these allowances before observing any data.
Write $s_0=R$, $\lambda=\max\{1,\log_2R\}$, and $\bar\eps=\min\{\eps,1/2\}$.
For an integer $k\ge1$, $\DS_k(\cH)$ is the supremum of the numbers of coordinates of witnesses in which every vector has at least $k$ distinct $i$-neighbors within the witness for each coordinate $i$. The value is zero if no such positive-dimensional witness exists. Set $\ell_t=\lceil R/2^{t+1}\rceil$ and $d_t=\DS_{\lceil\ell_t/2\rceil}(\cH)$, the dimension used by the stage's list learner. Define
\[
\alpha_t=\delta2^{-t-3},\quad u_t=1+\ln(1/\alpha_t),\quad
c_t=s_{t-1}(d_t+u_t).
\]
Here $u_t$ records the confidence overhead. The weight $c_t$ combines the list learner's sample requirement with the cost of obtaining positive responses from lists of width $s_{t-1}$. Allocate the total error allowance $\bar\eps$ by
\[
\eta_t=\frac{\bar\eps\sqrt{c_t}}{\sum_{j=1}^T\sqrt{c_j}}.
\]
More expensive stages receive larger error allowances, while the geometric $\alpha_t$ allocation matches the shrinking list sizes.

\textit{Collect examples and update the lists.}
Stage $t$ needs $n_t=\lceil100(d_t+u_t)/\eta_t\rceil$ accepted examples and receives a fixed block of
$m_t^{\mathrm{raw}}=\lceil16s_{t-1}(n_t+u_t)\rceil$ raw rounds. Predict uniformly from the preceding lists throughout this block. Keep the first $n_t$ positive examples and ignore later ones. If collection is incomplete at the end of the block, return any fixed hypothesis. Otherwise apply the deterministic list learner and pad its lists to size $s_t$ for the next stage. The chosen $n_t$ gives this update failure probability at most $\alpha_t$. After the final stage, output the singleton-list predictor.

\begin{theorem}[Refined cascade guarantee]\label{thm:cascade}
For every class with $B<\infty$ and every $\eps,\delta\in(0,1)$,
\[
\mB_\cH(\eps,\delta)
=O\left(\frac{B\lambda^2+R(1+L)}{\eps}\right).
\]
\end{theorem}
\proofref{app:cascade}

This theorem supplies the validity and total budget of the list-shrinking module. The calculation uses $k\DS_k(\cH)\le B$ and the geometrically decreasing stage widths. The sharper list input removes one dimension logarithm from the earlier list-PAC bound, but obtaining its labeled samples still multiplies the confidence cost by width. Appendix~\ref{app:list-compatibility} explains this distinction. Within its stated confidence range, Theorem~\ref{thm:interval} independently shows a matching width-times-confidence order on a worst-case family.

\subsection{Restrict the Class After One Positive}\label{sec:fiber}
The second module first predicts uniformly from $\{h(x):h\in\cH\}$ at each arriving instance, for at most $\lceil R\ln(2/\delta)\rceil$ rounds. If a positive occurs at $(x,y)$, retain only
$\cH_{x=y}=\{h\in\cH:h(x)=y\}$.
Store the first positive and ignore the remaining observations in the reserved first phase. Then run the refined cascade on fresh rounds for this surviving class, with accuracy parameter $\eps$ and failure probability $\delta/2$. If there is no positive in the allotted phase, return a fixed hypothesis. The first-phase budget also allows failure probability at most $\delta/2$; Theorem~\ref{thm:structural} certifies the combined guarantee.

The positive response supplies exactly the agreement needed for an anchor: every surviving hypothesis gives the same label at the observed point. That point can serve as an additional zero-multiplicity anchor for the surviving class's ordinary witnesses. Appending this common entry preserves the witness's total multiplicity and all required neighbor relations. The following lemma records the resulting dimension and width bounds used in the second phase.

\begin{lemma}[Positive-response restriction]\label{lem:fiberization}
For every point $x$ and realized label $y$,
\[
\BDS(\cH_{x=y})\le A,\qquad \wid(\cH_{x=y})\le A+1.
\]
\end{lemma}
\proofref{app:fiberization}

\begin{theorem}[Anchored structural upper bound]\label{thm:structural}
For $1\le A<\infty$, let $\lambda_A=\max\{1,\log_2(A+1)\}$. For all $\eps,\delta\in(0,1)$, the two-phase module satisfies
\[
\mB_\cH(\eps,\delta)=O\left(
R\ln\frac2\delta+
\frac{A\lambda_A^2+(A+1)(1+\ln(2/\delta))}{\eps}\right).
\]
In particular, $\mB_\cH(\eps,1/4)=O(R+(A+1)\log^2(A+2)/\eps)$.
\end{theorem}
\proofref{app:structural}

The first phase makes a true label available, and Lemma~\ref{lem:fiberization} converts that label into a dimension bound. The cascade then supplies the PAC guarantee. This explains why width can enter additively at constant confidence even when $B$ is much larger than $A$.

\section{Conclusion}\label{sec:conclusion}
The sample complexity of realizable bandit classification depends on more than how many alternatives a class contains. It also depends on where those alternatives can be distinguished and whether evidence against them can be shared. The anchored dimension separates globally testable labels from uncertainty concentrated on rare instances. PAB classes attain the resulting baseline through private elimination and block localization. The affine multiplexer exhibits the opposite behavior: negative evidence remains group-specific, producing a confidence direct sum until rank exhaustion becomes preferable.

The matching profile and the identical-parameter separation rule out a universal additive-confidence rate linear in the bandit DS dimension. They suggest a concrete next question: which class structure determines the extent to which confidence evidence can be reused? On the upper-bound side, positive-response restriction and refined cascade accounting give a common framework against which such structure can be measured. Removing the remaining logarithmic losses at constant confidence is a separate natural question.

\clearpage
\bibliography{references}
\clearpage
\appendix
\section{Proof Roadmap and Result Index}\label{app:roadmap}
The main text presents operations, guarantees, and the mechanisms linking them. The appendices prove those guarantees in dependency order. Table~\ref{tab:proof-index} gives a direct route from each main-text result to its proof. Figures~\ref{fig:structural-dependencies} and~\ref{fig:confidence-dependencies} show the principal proof chains. An arrow denotes mathematical support, possibly through an omitted intermediate lemma, not a step executed by the learner. The proof introductions specify the immediate dependencies.

\begin{table}[htbp]
\centering
\renewcommand{\arraystretch}{1.18}
\begin{tabular}{p{0.35\textwidth}p{0.39\textwidth}p{0.16\textwidth}}
\toprule
Main-text result & Property certified & Proof\\
\midrule
Theorem~\ref{thm:threepart} & Structural baseline for all learners & \ref{app:threepart}\\
Theorem~\ref{thm:zero} & Exact randomized elimination rate & \ref{app:zero}\\
Theorem~\ref{thm:pab} & PAB dimensions and algorithm & \ref{app:pab}\\
Theorem~\ref{thm:interval} & Separate confidence cost across groups & \ref{app:interval}\\
Proposition~\ref{prop:am-upper} & Two multiplexer learning procedures & \ref{app:am-upper}\\
Theorem~\ref{thm:profile} & Matching confidence profile & \ref{app:profile}\\
Corollary~\ref{cor:separation} & Additive-rate and parameter obstructions & \ref{app:consequences}\\
Theorem~\ref{thm:cascade} & Refined ListCascade budget & \ref{app:cascade}\\
Lemma~\ref{lem:fiberization} & Dimension after a positive response & \ref{app:fiberization}\\
Theorem~\ref{thm:structural} & Two-phase structural upper bound & \ref{app:structural}\\
\bottomrule
\end{tabular}
\caption{Main-text guarantees and their proof locations. A proof location can include the supporting lemmas needed immediately before the named result.}\label{tab:proof-index}
\end{table}

Appendix~\ref{app:foundations} establishes the shared structural, counting, and conversion tools. Appendix~\ref{app:lower} proves the baseline. Appendix~\ref{app:structured} verifies the PAB learner. Appendix~\ref{app:am} proves the confidence lower bound and both matching algorithms. Appendix~\ref{app:upper} derives the list input from external results before proving our bandit reductions. Appendix~\ref{app:supplement} gives supplementary comparisons and endpoints that are not premises of the main chain.

\begin{figure}[p]
\centering
\IfFileExists{figures/structural_dependencies.pdf}{\includegraphics[width=\textwidth]{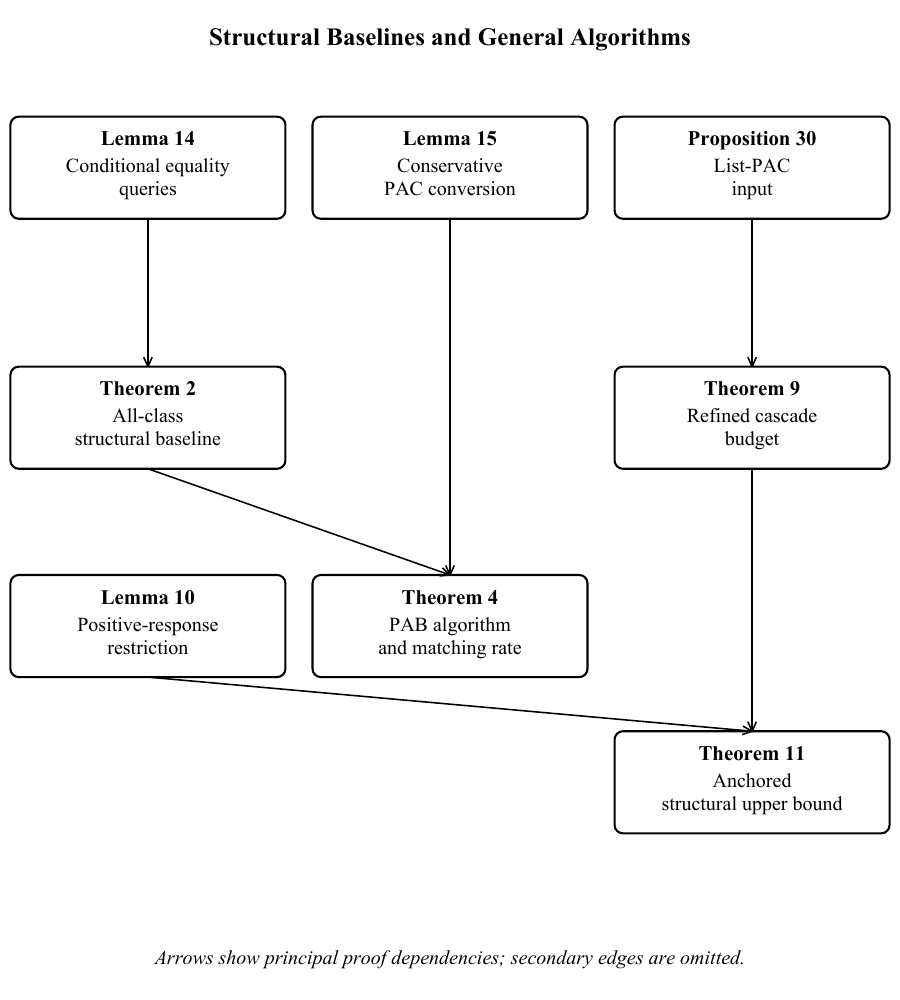}}{\fbox{Dependency diagram generated by build.py}}
\caption{Principal structural and algorithmic dependencies. The equality-query argument supports the baseline. The conservative conversion supports the PAB learner. List collection and allocation support the general cascade, which combines with positive-response restriction to give the anchored upper bound.}\label{fig:structural-dependencies}
\end{figure}
\begin{figure}[p]
\centering
\IfFileExists{figures/confidence_dependencies.pdf}{\includegraphics[width=\textwidth]{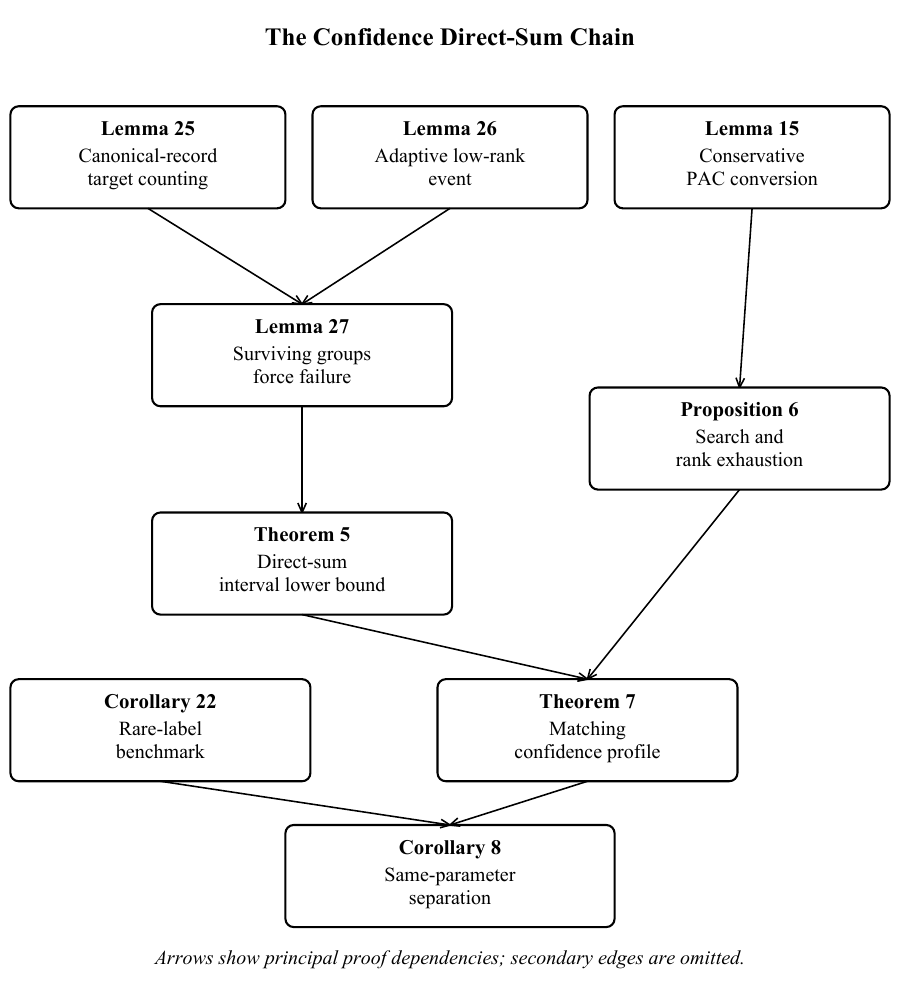}}{\fbox{Dependency diagram generated by build.py}}
\caption{Principal dependencies for the confidence separation. The lower chain counts targets on a canonical record and controls adaptive low-rank queries. The algorithmic chain uses conservative conversion. Their meeting point is the confidence profile; the PAB rare-label benchmark then supplies the identical-parameter comparison.}\label{fig:confidence-dependencies}
\end{figure}
\clearpage

\section{Structural and Probabilistic Foundations}\label{app:foundations}
We first establish the structural facts used in the baseline, then the equality-query counting tool used in its lower bound. The last lemma converts conservative mistake guarantees into PAC guarantees for the algorithms of Sections~\ref{sec:pab} and~\ref{sec:am}. These tools have different roles: equality-query counting limits what any learner can infer, whereas monotone conversion certifies particular algorithms.

We write a witness as $(S,N,V)$, where $S=(S_1,\ldots,S_d)$ is its instance sequence, $N=(N_1,\ldots,N_d)$ is its multiplicity vector, and $V\subseteq\cH|_S$ is its finite restriction family. This is the witness of Definition~\ref{def:dimensions}, with $S_i$ denoting its $i$th instance.

\subsection{Coordinates, Anchors, and Width}
\begin{lemma}[Distinctness of active coordinates]\label{lem:coordinates}
In a witness, an instance at a positive-multiplicity coordinate cannot occur at any other coordinate of the sequence. In particular, positive coordinates are distinct from one another and from every anchor instance.
\end{lemma}
\role{This makes the hard distribution in Theorem~\ref{thm:anchored-lower} well-defined. It also controls the number of available coordinates in dimension computations.}
\begin{proof}
Suppose $S_i=S_j$ with $i\ne j$ and $N_i>0$. Every restriction vector $v\in V$ has $v_i=v_j$. An $i$-neighbor $w$ must have $w_j=v_j$, because $j\ne i$. It also satisfies $w_i=w_j$, because the underlying instances coincide. Thus $w_i=v_i$, contradicting the definition of an $i$-neighbor. Since $V$ is nonempty and the positive multiplicity requires at least one such neighbor for every $v$, no such repetition is possible.
\end{proof}

\begin{lemma}[Dimension relations]\label{lem:dimensions}
Every nonempty class satisfies
\[
\max\{A,R-1\}\le B\le A+R-1.
\]
Moreover, $A>0$ if and only if there are two hypotheses agreeing at one point and disagreeing at another. For the class of all functions on $d\ge1$ points with $q\ge1$ labels,
$B=d(q-1)$, $A=(d-1)(q-1)$, and $R=q$.
\end{lemma}
\role{This proves \eqref{eq:chain-intro}, the endpoint criterion in Section~\ref{sec:baseline}, and the width estimate used after positive-response restriction.}
\begin{proof}
\textit{Projection gives $A\le B$.}
Take an anchored witness and delete all its zero-multiplicity coordinates. If its value is zero there is nothing to prove. Otherwise the remaining sequence has at least one coordinate. For each projected vector, choose one preimage. Its neighbors in a retained coordinate agree on all deleted coordinates and differ only in that retained coordinate. They remain distinct after projection and still are neighbors of the projected vector. The projected family is therefore an ordinary witness with the same total positive multiplicity. Taking suprema proves $A\le B$.

\textit{A single point gives $R-1\le B$.}
If $r$ distinct labels occur at a point $x$, choose a hypothesis realizing each label. Their restrictions to $(x)$ form a witness with multiplicity $r-1$. The inequality follows by taking the supremum over $x$. When $r=1$ the required inequality is simply $B\ge0$.

\textit{Removing one multiplicity gives $B\le A+R-1$.}
For an ordinary witness on one coordinate, its value is at most $R-1$. For a witness on at least two coordinates, choose any coordinate $a$ and replace $N_a$ by zero. The same restriction family becomes an anchored witness, so $\sum_{i\ne a}N_i\le A$. The base vector and its $a$-neighbors realize $N_a+1$ distinct labels at $S_a$, giving $N_a\le R-1$. Thus its original value is at most $A+R-1$. The argument applies to every finite witness, including when a supremum is infinite.

\textit{Positive anchored dimension is equivalent to agreement and disagreement.}
If $h,h'$ agree at $x$ and disagree at $x'$, their restrictions to $(x,x')$ give multiplicities $(0,1)$. Conversely, take a positive coordinate $i$ and an anchor coordinate $a$ in an anchored witness of positive value. A pair of $i$-neighbors agrees at $S_a$ and differs at $S_i$. Representatives in $\cH$ supply the required pair of hypotheses.

\textit{Full product classes.}
The full product on the $d$ domain points has $q-1$ neighbors in every coordinate, so $B\ge d(q-1)$. Lemma~\ref{lem:coordinates} allows at most $d$ positive coordinates, each with multiplicity at most $q-1$, proving the reverse inequality. An anchored witness has at most $d-1$ positive coordinates, giving $A\le(d-1)(q-1)$. For $d\ge2$, make one coordinate of the full product an anchor to attain this bound. For $d=1$, Lemma~\ref{lem:coordinates} precludes positive anchored value. The case $q=1$ has no positive multiplicities and satisfies all formulas directly.
\end{proof}

\subsection{Counting Equality Queries on a Conditional Fiber}
\begin{lemma}[Conditional equality-query bound]\label{lem:equality}
Let $V$ be a finite nonempty subset of $\cY^d$, let $V^\star$ be uniform on $V$, and fix a coordinate $i$. Write $U=V^\star_{-i}$. For a value $u$ of the outside coordinates, define the fiber
\[
Q_i(u)=\{v_i:v\in V,\ v_{-i}=u\}.
\]
If, for some $N_i\in\mathbb N_0$, every $v\in V$ has at least $N_i$ distinct $i$-neighbors in $V$, then $|Q_i(u)|\ge N_i+1$ for every realized $u$.

Let $Z$ be auxiliary information conditionally independent of $V_i^\star$ given $U$. Conditional on $(U,Z)=(u,z)$, consider a protocol whose only further access to $V_i^\star$ is through equality queries, with at most $c(u,z)$ queries on every execution path. Its random seed is independent of the target conditional on $(U,Z)$. Once $(u,z)$ and the seed are fixed, each queried label, the decision to stop, and the final guess $\widehat v_i$ are deterministic functions of the previous equality-query answers. Almost surely in $(U,Z)$,
\[
\Pr[\widehat v_i=V_i^\star\mid U=u,Z=z]
\le\min\left\{1,\frac{c(u,z)+1}{|Q_i(u)|}\right\}.
\]
\end{lemma}
\role{This replaces an assumption about a learner's guessing strategy by a counting argument valid for arbitrary adaptive strategies. Theorem~\ref{thm:anchored-lower} supplies the auxiliary information and verifies the no-side-information condition.}
\begin{proof}
For a fixed realized value $u$, the vectors in $V$ with outside coordinates $u$ are in one-to-one correspondence with $Q_i(u)$. Uniformity of $V^\star$ therefore makes $V_i^\star$ uniform on this set conditional on $U=u$. Choose any $v\in V$ with $v_{-i}=u$. Its $N_i$ required neighbors belong to $V$ and have the same outside coordinates $u$. Their labels at coordinate $i$ are distinct from one another and from $v_i$. These labels all belong to $Q_i(u)$, proving $|Q_i(u)|\ge N_i+1$. Conditional independence means that revealing $Z=z$ does not change this uniform distribution.

Now fix $(u,z)$ and a seed, and write $c=c(u,z)$. If $c+1\ge |Q_i(u)|$, the asserted bound is just the fact that a probability is at most one. Suppose instead that $c+1<|Q_i(u)|$. Follow the deterministic protocol while answering every equality query negatively. After at most $c$ queries, at least one label in the fiber remains unqueried. This partial record is therefore realizable by a target in the fiber. The pathwise cap forces the protocol to stop without making more than $c$ queries. This observation avoids assigning any significance to an unrealizable branch of a decision tree.

Let $T\subseteq Q_i(u)$ be the set of fiber labels queried on this all-negative execution. Repeated queries and queries outside the fiber do not enlarge $T$, so $|T|\le c$. If a target ever yields a first positive response, its earlier responses were negative. Its query at that first positive therefore lies on the all-negative execution, so the target label belongs to $T$. Every target outside $T$ produces the entire all-negative record. All operations and the final guess are identical for these targets because the seed is fixed and there is no other source of target information. That guess can be correct for at most one of them. Even granting success on every label in $T$, the protocol succeeds on at most $|T|+1\le c+1$ fiber labels. Dividing by $|Q_i(u)|$ and averaging over the conditionally independent seed proves the claim.
\end{proof}

The query cap is imposed on every branch after the auxiliary information is fixed. A realized stopping count cannot replace it: a learner stopping at its first correct guess can always succeed, even on paths with few queries. In our lower-bound application, the instance sequence fixes the number of opportunities at the coordinate before any labels are examined.

\subsection{Conservative Learning and Final-Predictor Risk}\label{app:monotone}
\begin{lemma}[Monotone online-to-PAC conversion]\label{lem:monotone}
Let $M\in\mathbb N_0$ and $\eps,\delta\in(0,1)$. Fix a target $h^\star\in\cH$ and a distribution $\cD$ on $\cX$. Suppose a bandit algorithm maintains a predictor $h_t$ before round $t$, predicts $h_t(x_t)$, and returns $h_{m+1}$. For every fixed learner seed and every instance sequence labeled by $h^\star$, assume that its error sets
$E_t=\{x:h_t(x)\ne h^\star(x)\}$ satisfy $E_{t+1}\subseteq E_t$ and it receives at most $M$ negative responses. On $m$ i.i.d. draws from $\cD$, independent of the seed,
\[
\Pr[\err_\cD(h_{m+1})>\eps]\le
\Pr[\operatorname{Bin}(m,\eps)\le M].
\]
In particular, $m\eps\ge\max\{2M,8\ln(1/\delta)\}$ gives failure probability at most $\delta$.
\end{lemma}
\role{This is the shared probabilistic step behind the PAB guarantee and both within-group learning procedures for the multiplexer. Each application separately proves the two pathwise hypotheses.}
\begin{proof}
Fix a target and a learner seed. Include the seed in the history. We compare the actual negative responses with independent Bernoulli trials of success probability $\eps$. The actual responses need not be independent, so we construct the comparison explicitly.

Let $\mathcal F_{t-1}$ be the history before round $t$, and set $r_t=\cD(E_t)$, including $r_{m+1}$ for the final predictor. If $Z_t$ is the indicator of a negative response, the independence of $x_t$ from the history gives
$\Pr[Z_t=1\mid\mathcal F_{t-1}]=r_t$.
Introduce independent auxiliary variables $V_t$, uniform on $[0,1]$ and independent of the entire learning process. Define
\[
U_t=\begin{cases}
r_tV_t,&Z_t=1,\\
r_t+(1-r_t)V_t,&Z_t=0.
\end{cases}
\]
Let $\mathcal F'_{t-1}=\sigma(\mathcal F_{t-1},V_1,\ldots,V_{t-1})$ be the enlarged history. Conditional on this history, the first interval is selected with probability $r_t$ and the second with probability $1-r_t$. On its selected interval $U_t$ is uniform. For $0<r_t<1$ and $0\le u\le1$, its conditional distribution function is
\[
\Pr[U_t\le u\mid\mathcal F'_{t-1}]
=\begin{cases}
r_t(u/r_t)=u,&u\le r_t,\\
r_t+(1-r_t)(u-r_t)/(1-r_t)=u,&u>r_t.
\end{cases}
\]
The previous auxiliary variables do not affect the conditional law of $Z_t$, because the auxiliary sequence is independent of the learning process and does not enter its update rule. The endpoint cases $r_t=0$ and $r_t=1$ also give a uniform variable, with the zero-probability branch ignored. Each $U_j$ with $j<t$ is measurable with respect to $\mathcal F'_{t-1}$. Thus successive conditioning gives, for all $a_1,\ldots,a_m\in[0,1]$,
\[
\Pr[U_1\le a_1,\ldots,U_m\le a_m]
=\prod_{t=1}^m a_t.
\]
This proves that $U_1,\ldots,U_m$ are independent uniform variables. Except on probability-zero endpoints, $Z_t=\ind{U_t\le r_t}$.

Set $W_t=\ind{U_t\le\eps}$. The sum of the $W_t$ has distribution $\operatorname{Bin}(m,\eps)$. If $r_{m+1}>\eps$, the nesting of error sets implies $r_t>\eps$ for every $t\le m$. On this event, $W_t\le Z_t$ for each round, and therefore
\[
\sum_{t=1}^mW_t\le\sum_{t=1}^mZ_t\le M.
\]
This event inclusion proves the first assertion. For the second, let $\mu=m\eps$. The assumption $\mu\ge2M$ gives
\[
\Pr[\operatorname{Bin}(m,\eps)\le M]
\le\Pr[\operatorname{Bin}(m,\eps)\le\mu/2]
\le e^{-\mu/8}\le\delta,
\]
by the binomial lower-tail bound. To record the precise estimate used here and in the later algorithms, let $X\sim\operatorname{Bin}(m,p)$, with $p\in[0,1]$, and put $\nu=mp$. Exponential Markov's inequality gives
\begin{equation}\label{eq:binomial-half-tail}
\begin{aligned}
\Pr[X\le\nu/2]
&\le 2^{\nu/2}\E[2^{-X}]
=2^{\nu/2}(1-p/2)^m\\
&\le\exp\left(-\frac{1-\ln2}{2}\nu\right)
\le e^{-\nu/8}.
\end{aligned}
\end{equation}
Here $1-z\le e^{-z}$ gives the first exponential bound, and $\ln2<3/4$ gives the last one. Applying this with $p=\eps$ justifies the preceding calculation. Since the argument holds for every fixed learner seed with the same constants, averaging proves the randomized version.
\end{proof}

The conversion uses the maintained final predictor. A training mistake bound alone does not control an arbitrary output rule. Nor does finite $B$ itself supply a finite adversarial mistake bound; Proposition~\ref{prop:thresholds} gives an explicit example.

\section{Proofs of the Structural Baseline}\label{app:lower}
We prove the three lower-bound mechanisms separately, combine them, and then verify the examples used in Section~\ref{sec:baseline}. All hard instances choose a target before the interaction. Random target selection is used only to show that some fixed target defeats each learner.

\Needspace{9\baselineskip}
\subsection{Branching on Rare Instances}
\begin{theorem}[Anchored lower bound]\label{thm:anchored-lower}
For every class, $0<\eps\le1/16$, and $0<\delta\le1/4$,
$\mB_\cH(\eps,\delta)\ge A/(32\eps)$.
If $A=\infty$, the sample complexity is infinite.
\end{theorem}
\role{This supplies the $A/\eps$ term of Theorem~\ref{thm:threepart}. It uses the distinctness of coordinates from Lemma~\ref{lem:coordinates} and conditional counting from Lemma~\ref{lem:equality}.}
\begin{proof}
\textit{Choose the target prior and the hard distribution.}
Fix an anchored witness $(S,N,V)$ of positive value $M$, and choose an anchor index $a$. Let $P=\{i:N_i>0\}$, so $M=\sum_{i\in P}N_i$. For every $v\in V$, choose one representative $h_v\in\cH$ with restriction $v$. Draw the target by taking $V^\star$ uniformly from $V$ and setting $h^\star=h_{V^\star}$.

Define a distribution independently of the target by
\[
\cD(S_a)=1-8\eps,\qquad
\cD(S_i)=p_i:=\frac{8\eps N_i}{M}\quad(i\in P),
\]
and give all other points zero mass. Lemma~\ref{lem:coordinates} ensures that the displayed points are distinct. The masses sum to one and are nonnegative because $8\eps\le1/2$.

\textit{Bound the chance of predicting one rare coordinate.}
Consider any learner using an integer $m\le M/(32\eps)$ rounds. Let $C_i$ count the occurrences of $S_i$ in its instance sequence. Then
\[
\E C_i=mp_i\le\frac{N_i}{4}.
\]
For each $i\in P$, condition on the entire instance sequence and give the learner all coordinates $U=V^\star_{-i}$ for free. The instance sequence is independent of the target, so revealing it preserves uniformity of the conditional fiber. All points in the distribution's support belong to $S$. Responses at points other than $S_i$ can therefore be computed from the revealed outside coordinates and the learner's predictions. Only a visit to $S_i$ tests its unknown label. There are at most $C_i$ such visits on every path. The original learner is included among the protocols allowed this extra information: it can simply ignore it. Lemma~\ref{lem:equality} therefore applies to its final value $\hat h(S_i)$, including an improper final value. With $Q_i(U)$ denoting the fiber from that lemma, it gives
\[
\Pr[\hat h(S_i)=V_i^\star\mid U,x_1,\ldots,x_m]
\le\min\left\{1,\frac{C_i+1}{|Q_i(U)|}\right\}
\le\frac{C_i+1}{N_i+1}.
\]
Here the conditional probability already averages over the independent learner seed. The auxiliary information in the lemma is the instance sequence, not the observed labels.

Averaging this conditional bound over the outside coordinates and the instance sequence yields
\[
\Pr[\hat h(S_i)=V_i^\star]
\le\frac{\E C_i+1}{N_i+1}
\le\frac{N_i/4+1}{N_i+1}.
\]
Consequently,
\[
\Pr[\hat h(S_i)\ne V_i^\star]
\ge\frac{3N_i}{4(N_i+1)}\ge\frac38,
\]
where the last inequality uses $N_i\ge1$.

\textit{Turn coordinate errors into PAC failure.}
Let $Z=\sum_{i\in P}p_i\ind{\hat h(S_i)\ne V_i^\star}$ be the contribution to risk from the rare points. Since $\sum_{i\in P}p_i=8\eps$, its range is $[0,8\eps]$. Linearity of expectation, without any independence assumption between coordinate errors, gives
\[
\E Z=\sum_{i\in P}p_i\Pr[\hat h(S_i)\ne V_i^\star]
\ge\frac38\sum_{i\in P}p_i=3\eps.
\]
Write $q=\Pr[Z>\eps]$. On $\{Z\le\eps\}$ the variable is at most $\eps$, and on $\{Z>\eps\}$ it is at most $8\eps$. Splitting the expectation between the two events yields
\[
3\eps\le\E Z\le\eps(1-q)+8\eps q,
\qquad q\ge\frac27>\frac14\ge\delta.
\]
The full risk is at least $Z$. This failure probability is averaged over a finite target prior. More explicitly, if $q_v$ is the failure probability under the fixed target $h_v$, then $|V|^{-1}\sum_{v\in V}q_v\ge2/7$. At least one $q_v$ is therefore greater than $\delta$. The distribution and this representative can depend on the learner, as permitted in a worst-case lower bound, but both are fixed before its actual random interaction. Thus every integer $m\le M/(32\eps)$ is insufficient.

If $A=0$, the asserted inequality is vacuous. If $0<A<\infty$, all witness values are nonnegative integers. A finite positive supremum of such values is attained, so we may take $M=A$. If $A=\infty$, for each proposed finite sample budget $m$ choose a finite witness with $M\ge32\eps m$. The same construction defeats that budget. This establishes all cases of the theorem.
\end{proof}

\subsection{Label Search and the Exact Endpoint}\label{app:zero}
\begin{lemma}[Point-mass width bound]\label{lem:width-lower}
For every class and every $\eps,\delta\in(0,1)$,
$\mB_\cH(\eps,\delta)\ge(1-\delta)R-1$.
\end{lemma}
\role{This is the width term in Theorem~\ref{thm:threepart} and the lower half of the exact formula in Theorem~\ref{thm:zero}.}
\begin{proof}
Fix a point realizing $r$ labels and choose a target representative for each label. Concentrate the instance distribution at this point and choose the target label uniformly. Every response is an equality test of the same unknown label. After fixing the learner's seed, the all-negative-path argument of Lemma~\ref{lem:equality} bounds success by $(m+1)/r$. The final risk on this point mass is zero for a correct label and one for an incorrect label. Since $\eps<1$, PAC success requires a correct final label. Averaging over seeds does not change the bound.

If a learner succeeds with probability at least $1-\delta$ for each fixed target, its average success over these $r$ representatives is also at least $1-\delta$. Hence
\[
1-\delta\le\frac{m+1}{r},\qquad m\ge(1-\delta)r-1.
\]
The inequality holds for every point and every learner with the claimed guarantee. Taking the supremum over the numbers of realized labels gives the result.
\end{proof}

\role{The following proof certifies the randomized elimination procedure in Section~\ref{sec:baseline}. It uses Lemma~\ref{lem:dimensions} to identify the class structure, and Lemma~\ref{lem:width-lower} for optimality.}
\begin{proof}\textbf{of Theorem~\ref{thm:zero}}
If two distinct hypotheses agreed somewhere, they would disagree somewhere else and Lemma~\ref{lem:dimensions} would give $A>0$. Thus every evaluation map $h\mapsto h(x)$ is injective. At every point it realizes exactly $|\cH|$ labels, so $|\cH|=R<\infty$.

Let $m=\max\{0,\lceil(1-\delta)R\rceil-1\}$. Since $0<1-\delta<1$, we have $0\le m\le R-1$. Choose a uniformly random ordering $h_{(1)},\ldots,h_{(R)}$ of the class. On round $t\le m$, predict $h_{(t)}(x_t)$ until a positive occurs. Injectivity means that a positive identifies the target globally. After identification the learner may ignore the remaining rounds and output that hypothesis. If all $m$ responses are negative, choose uniformly from the $R-m$ untested hypotheses.

For every fixed target, its position in the random order is uniform. The probability of exact identification or a correct final guess is
\[
\frac mR+\frac{R-m}{R}\frac1{R-m}=\frac{m+1}{R}\ge1-\delta.
\]
On this event the error is zero under every distribution. This proves sufficiency, including $m=0$ and $R=1$. Conversely, let $m'$ be any sample budget attaining the required guarantee. Lemma~\ref{lem:width-lower} and integrality of $m'$ give
$m'\ge\max\{0,\lceil(1-\delta)R-1\rceil\}
=\max\{0,\lceil(1-\delta)R\rceil-1\}$.
\end{proof}

\subsection{The Ordinary Confidence Cost}
\begin{lemma}[Two-target confidence bound]\label{lem:confidence-lower}
If $A>0$, $0<\eps\le1/4$, and $0<\delta\le1/4$, then
$\mB_\cH(\eps,\delta)\ge L/(8\eps)$.
\end{lemma}
\role{This supplies the confidence term in Theorem~\ref{thm:threepart} and the very-high-confidence lower bound in Theorem~\ref{thm:profile}.}
\begin{proof}
By Lemma~\ref{lem:dimensions}, choose hypotheses $h_0,h_1$ agreeing at $x_0$ and differing at $x_1$. Put masses $1-2\eps$ and $2\eps$ at these points, and choose the target uniformly from the pair. On the event that $x_1$ never occurs, the entire transcript has the same distribution under both targets. This remains true for randomized adaptive predictions. The final value at $x_1$ can match at most one of the two target labels. An incorrect value contributes $2\eps>\eps$ to risk. The prior-average failure probability is therefore at least
\[
\frac12(1-2\eps)^m\ge\frac12e^{-4\eps m}.
\]
For the displayed inequality, put $u=2\eps\le1/2$. The elementary estimate
\[
-\ln(1-u)=\int_0^u\frac{1}{1-v}\,dv
\le\frac{u}{1-u}\le2u
\]
gives $(1-2\eps)^m\ge e^{-4\eps m}$.
If every target had failure probability at most $\delta$, it would follow that
\[
m\ge\frac{\ln(1/(2\delta))}{4\eps}
\ge\frac{\ln(1/\delta)}{8\eps},
\]
because $L\ge\ln4=2\ln2$. This proves the claim.
\end{proof}

\subsection{Combining the Costs}\label{app:threepart}
\role{This is the proof of the direct guarantee for the baseline module. Its three ingredients are Theorem~\ref{thm:anchored-lower}, Lemma~\ref{lem:width-lower}, and Lemma~\ref{lem:confidence-lower}. They may use different hard distributions, which is sufficient because each lower-bounds the same worst-case sample complexity.}
\begin{proof}\textbf{of Theorem~\ref{thm:threepart}}
For $R\ge2$ and $\delta\le1/4$, the width bound implies
\[
\mB_\cH\ge\frac{3R}{4}-1\ge\frac{R-1}{4}.
\]
For $R=1$, all hypotheses are the same function, $A=0$, and the asserted lower bound is zero. In the remaining cases the preceding results give
\[
\mB_\cH\ge\max\left\{
\frac{A}{32\eps},\frac{\ind{A>0}L}{8\eps},\frac{R-1}{4}\right\}.
\]
The maximum of three nonnegative numbers is at least one third of their sum. Replacing the larger coefficients by $1/32$ gives the stated coefficient $1/96$. The infinite-dimensional case follows from Theorem~\ref{thm:anchored-lower}. In addition, if $R\ge2$, using $R-1\ge R/2$ yields the useful variant
$\mB_\cH\ge(A/\eps+\ind{A>0}L/\eps+R)/192$.
\end{proof}

\subsection{The Benchmark Counterexamples}\label{app:benchmarks}
\begin{proposition}[Constant and pair-anchored classes]\label{prop:benchmarks}
The class of $K$ constant functions has $(A,B,R)=(0,K-1,K)$ and the exact complexity of Theorem~\ref{thm:zero}. For $K\ge3$, the class $\cH_{\mathrm{anc}}$ of Section~\ref{sec:baseline} has $(A,B,R)=(1,K-1,K)$ and
$\mB_{\cH_{\mathrm{anc}}}(\eps,\delta)\le K-2+\lceil L/\eps\rceil$
for all $\eps,\delta\in(0,1)$.
\end{proposition}
\role{This verifies the examples motivating the anchored parameter. The pair-anchored example has a pair agreeing at $a$, not a common anchor for all its hypotheses.}
\begin{proof}
Constant functions disagree everywhere, so $A=0$ by Lemma~\ref{lem:dimensions}, and $R=K$. The same lemma gives $B=K-1$, and Theorem~\ref{thm:zero} applies.

For $\cH_{\mathrm{anc}}$, the point $b$ realizes all $K$ labels. Hence $R=K$ and $B\ge K-1$. A witness with two positive coordinates would need neighbors differing at $a$ and agreeing at $b$. The restrictions $(1,1),(1,2),(j,j)$ have distinct values at $b$, so there is no such pair. Lemma~\ref{lem:coordinates} excludes additional or repeated active coordinates. A single-coordinate witness has multiplicity at most $K-1$. Thus $B=K-1$.

An anchored witness of positive value on this two-point domain must have one positive coordinate and the other point as anchor. With anchor $a$, the only repeated anchor label is $1$, shared by $(1,1),(1,2)$, so the positive multiplicity is at most one. With anchor $b$, no pair shares a label. The first pair realizes value one, giving $A=1$.

For the upper bound, test labels $3,\ldots,K$ in the first $K-2$ rounds, regardless of the instances. A positive at $j$ identifies $(j,j)$, and otherwise these hypotheses are all eliminated. For the next $r=\lceil L/\eps\rceil$ rounds, the candidates are $(1,1)$ and $(1,2)$. At $a$ predict $1$. On the first appearance of $b$, predict $2$; either response identifies the target within the pair. If $b$ never appears, output $(1,1)$. When $\cD(b)\le\eps$, this fallback has risk at most $\eps$. Otherwise its probability is at most
$(1-\cD(b))^r\le e^{-r\eps}\le\delta$.
\end{proof}

\section{Proofs for the Structured Algorithm}\label{app:structured}
This appendix verifies the representation and the three actions of Section~\ref{sec:pab}. The dimension calculation and the algorithm analysis are separate. Their conclusions meet in the matching-rate theorem.

\begin{lemma}[Dimensions of PAB classes]\label{lem:pab-dimensions}
For the PAB representation in Section~\ref{sec:pab},
$A=M$, $B=\max\{M,R-1\}$, and $R=q+1+w_{\max}$.
\end{lemma}
\role{This connects the algorithm's block mistake budget to the anchored lower bound. It uses Lemma~\ref{lem:coordinates} and the parameter relations of Lemma~\ref{lem:dimensions}.}
\begin{proof}
\textit{Count the labels available at one point.}
At a point in block $X_j$, the core realizes exactly the labels $A_{j,x}$. Other blocks take only the base label there. At a point outside all blocks, the core realizes only the base label. In both cases the private hypotheses add exactly $q$ distinct labels. Taking the supremum gives $R=q+1+w_{\max}$, also when there are no nonconstant blocks.

\textit{Confine a multi-coordinate witness to one block.}
Consider a witness on at least two coordinates with a positive multiplicity. A restriction of a private hypothesis has no single-coordinate neighbor: every other hypothesis differs from it at every point, hence at a second coordinate. Consequently no private restriction can belong to this witness. At any positive coordinate $x$, the required variation among core restrictions implies that $x$ lies in a block $X_j$ with a non-base label present in some restriction $v$. A representative of $v$ must belong to $\cH_j$, since no other core hypothesis takes a non-base label at $x$.

If another positive coordinate $x'$ lay outside $X_j$, a neighbor of $v$ at $x'$ would have to preserve its non-base value at $x$. Its representative would therefore also belong to $\cH_j$. Such a representative takes the base value at $x'$, just as $v$ does, so it cannot be a neighbor there. This contradiction shows that all positive coordinates lie in the same block. At any of these coordinates, all core labels belong to $A_{j,x}$, including labels supplied by other blocks, which are the base label. Distinct neighbors must use distinct labels other than the original vector's label. There can therefore be at most $|A_{j,x}|-1=w_{j,x}$ neighbors in that direction. Lemma~\ref{lem:coordinates} makes the positive coordinates distinct, so summing these bounds gives witness value at most $M_j\le M$.

\textit{Attain the anchored value and recover the ordinary dimension.}
Every positive anchored witness has at least two coordinates, so the preceding argument proves $A\le M$. For a block with $M_j>0$, use its full finite product class, the positive-width points of $X_j$, and an external anchor $a_j$. Give the positive points multiplicities $w_{j,x}$ and give $a_j$ multiplicity zero. The full product supplies all required neighbors and yields value $M_j$. Taking the supremum proves $A=M$. If all $M_j$ vanish, the upper bound already gives $A=0$.

Ordinary witnesses on at least two coordinates have value at most $M$ by the same argument. One-coordinate witnesses have value at most $R-1$. Thus $B\le\max\{M,R-1\}$. The reverse inequality follows from $B\ge\max\{A,R-1\}$ and $A=M$.
\end{proof}

\begin{lemma}[Private elimination and conservative block learning]\label{lem:pab-algorithm}
For a PAB class, the private-testing phase either identifies the target or places it in the core after at most $q$ rounds. Against a core target, the subsequent pointwise-update algorithm of Section~\ref{sec:pab-actions} has nested error sets and at most $M$ negative responses on every realizable sequence.
\end{lemma}
\role{These are exactly the deterministic hypotheses required by Lemma~\ref{lem:monotone}. They certify the operations described before Theorem~\ref{thm:pab}.}
\begin{proof}
When the algorithm tests $g_s(x)$, pointwise privacy makes this label equal to the target label if and only if the target is $g_s$. Thus a positive identifies that hypothesis, while a negative eliminates it. After all private tests fail, the target lies in $\cH_{\mathrm{core}}$.

If the target is $h_0$, every base prediction is correct and no update is ever needed. Otherwise the target differs from $h_0$ at some point and belongs to one unique nonconstant block class $\cH_{j^\star}$. Uniqueness follows from disjointness: a non-base value at a point of $X_{j^\star}$ cannot be supplied by another block. Before the first negative, predictions remain $h_0$. At that first negative point $x$, the target is non-base, so $x\in X_{j^\star}$ and its block is identified.

Fix an ordering of $A_{j^\star,x}$ at each active point, with the base label first. On a negative response, remove the current label and move to the next remaining label. The true label belongs to this list and is never removed. At most $|A_{j^\star,x}|-1=w_{j^\star,x}$ negatives can occur at $x$. The first locating negative is included in this count. Outside the active block the prediction remains $h_0$, which equals the target. Summing over the finite active block gives at most $M_{j^\star}\le M$ negatives.

Only the current mistaken point is modified on an update. Its new label is either still wrong or becomes correct. A correct current label never receives a negative response, and so is never changed. Every other point retains its previous label. These observations prove that the error set is nested after every update, including the first one that identifies the block.
\end{proof}

\subsection{The Direct Guarantee}\label{app:pab}
\role{Theorem~\ref{thm:pab} supports all three PAB modules. Lemma~\ref{lem:pab-dimensions} identifies their structural budget; Lemma~\ref{lem:pab-algorithm} and Lemma~\ref{lem:monotone} prove the upper bound; Theorem~\ref{thm:threepart} supplies the matching lower bound.}
\begin{proof}\textbf{of Theorem~\ref{thm:pab}}
\textit{Certify the algorithm and its budget.}
The parameter identities are Lemma~\ref{lem:pab-dimensions}. A finite $M$ is a nonnegative integer, since it is the supremum of integer block widths, so it is an admissible budget in Lemma~\ref{lem:monotone}. If a private test is positive, the output is exactly correct and remaining rounds may be ignored. Conditional on all $q$ tests being negative, the target is a core target. The next fixed block of observations is still i.i.d. from $\cD$, independently of the testing history. Discard that history except for the conclusion that the target is in the core.

For $M>0$, let $m_{\mathrm{core}}=\lceil\max\{2M,8L\}/\eps\rceil$ be the number of rounds in this fresh block. Lemma~\ref{lem:pab-algorithm} provides nested error sets and at most $M$ negatives. Since
\[
m_{\mathrm{core}}\eps\ge\max\{2M,8\ln(1/\delta)\},
\]
Lemma~\ref{lem:monotone} bounds the conditional failure probability by $\delta$. The core phase includes the wait for its first negative response; locating the block does not require any additional, uncounted phase. The total budget is at most $q+m_{\mathrm{core}}$. If $M=0$, every $w_{j,x}$ is zero. Each block class therefore contains only $h_0$, and the core is $\{h_0\}$. After the private tests no further sample is needed.

\textit{Match the structural lower bound.}
Suppose first that $M>0$. Then $M\ge1$ and $R\ge2$. Since $q\le R-1$ and a ceiling contributes at most one round,
\[
q+m_{\mathrm{core}}
\le R-1+\frac{2M+8L}{\eps}+1
=O\left(R-1+\frac{M+L}{\eps}\right).
\]
The last step uses $M/\eps\ge1$ in the standard range to absorb the extra round.
Theorem~\ref{thm:threepart} with $A=M$ supplies the reverse inequality with a universal constant. If $M=0$ and $R\ge2$, the upper bound is $q=R-1$, and Lemma~\ref{lem:width-lower} gives at least $(R-1)/4$ rounds in the standard range. If $M=0$ and $R=1$, the class is a singleton and both the stated expression and sample complexity are zero. This covers every case of the matching-rate assertion.
\end{proof}

\begin{corollary}[Rare-label benchmark]\label{cor:rare}
For $\cG_K$ defined in Section~\ref{sec:pab} and $K\ge2$,
$(A,B,R)=(K-1,K-1,K)$ and
$\mB_{\cG_K}(\eps,\delta)=\Theta((K+L)/\eps)$ in the standard range.
\end{corollary}
\role{This provides the comparison class in Corollary~\ref{cor:separation}.}
\begin{proof}
Take the base function $(1,1)$, the single block $\{b\}$ with label set $[K]$, external anchor $a$, and no private hypotheses. Then $M=w_{\max}=K-1$. Theorem~\ref{thm:pab} gives the parameter identities and a rate of
$\Theta(K-1+(K-1+L)/\eps)$.
Since $K\ge2$ and $\eps<1$, this equals $\Theta((K+L)/\eps)$ with universal constants.
\end{proof}

\section{Proofs of the Confidence Separation}\label{app:am}
This appendix follows the lower-bound and algorithmic modules of Section~\ref{sec:am}. We first compute the construction's dimensions. Three lemmas then establish the lower-bound mechanism, after which we prove the interval bound. The two algorithms are analyzed separately before their guarantees are combined into the profile and its consequences.

\subsection{Dimensions of the Construction}
\begin{lemma}[Binary neighbor closure]\label{lem:binary-cube}
If a nonempty $W\subseteq\{0,1\}^s$ contains an $i$-neighbor of each of its vectors for every coordinate $i\in[s]$, then $W=\{0,1\}^s$ and $|W|=2^s$.
\end{lemma}
\role{This controls multi-coordinate witnesses within a single multiplexer group in Lemma~\ref{lem:am-dimensions}.}
\begin{proof}
In a binary alphabet, the only $i$-neighbor of a vector is the vector obtained by flipping its $i$th bit. Thus $W$ is closed under every individual bit flip. Starting from any member of the nonempty set, flip precisely the coordinates at which a desired binary vector differs. The resulting vector is in $W$. Every binary vector is therefore in $W$, proving both assertions. This includes $s=0$, where the space has one vector.
\end{proof}

\begin{lemma}[Multiplexer dimensions]\label{lem:am-dimensions}
For every $n\ge1$, the class $\AM$ has $A=B=2n-1$ and $R=2n$.
\end{lemma}
\role{This verifies the structural parameters in Section~\ref{sec:am-construction} and later makes the identical-parameter comparison possible. It uses Lemmas~\ref{lem:coordinates}, \ref{lem:dimensions}, and~\ref{lem:binary-cube}.}
\begin{proof}
\textit{One informative point supplies the width lower bound.}
For every nonzero vector $x$, choose a coordinate $i$ with $x_i=1$. The zero parameter has inner product zero with $x$, while the parameter that is one at coordinate $i$ and zero elsewhere has inner product one. Thus each group realizes both binary values at $x$. There are exactly $2n$ realized labels there, while the anchor realizes only $\bot$. Hence $R=2n$, and a one-coordinate witness at $x$ gives $B\ge2n-1$.

\textit{A multi-coordinate witness cannot combine branching from different groups.}
Consider an ordinary witness with restriction family $V$ on $s\ge2$ coordinates. Its coordinates are distinct by Lemma~\ref{lem:coordinates}, and none is the anchor, which carries no variation. Two neighbors agree on at least one other informative coordinate. Equality of their labels at that coordinate forces their group indices to coincide. Every neighbor of a vector therefore belongs to the same group as that vector. Within one group only the binary component can change, so there is at most one neighbor in each direction. The positive multiplicities must consequently all equal one.

Choose a group $g$ represented in $V$ and remove the fixed group index from its vectors. The resulting nonempty binary family is
\[
W_g=\{(b_1,\ldots,b_s):((g,b_1),\ldots,(g,b_s))\in V\}.
\]
For each member of $W_g$, the witness supplies a neighbor in every coordinate. The preceding argument puts each such neighbor in the same group, so it remains in $W_g$. Lemma~\ref{lem:binary-cube} gives $|W_g|=2^s$. Group $g$ has only $2^n$ parameters and hence at most $2^n$ distinct restrictions. Thus $2^s\le2^n$, or $s\le n$. The witness value is $s$, at most $n\le2n-1$. One-coordinate witnesses have value at most $R-1=2n-1$, so all ordinary witnesses satisfy this bound. The earlier construction attains it, proving $B=2n-1$.

Every hypothesis agrees at $a$. Append $a$ with multiplicity zero to any ordinary witness. Since it is a constant coordinate on the witness, all neighbor relations persist. In particular the value $2n-1$ is attained by an anchored witness. Now $A\le B$ from Lemma~\ref{lem:dimensions} gives $A=B$.
\end{proof}

\subsection{The Canonical Record and Its Surviving Targets}
In the next three lemmas, fix $0<\eps\le1/16$ and use the hard distribution $\cD(a)=1-4\eps$ and $\cD(x)=4\eps/(2^n-1)$ for $x\in V_n$. Independently draw $J$ uniformly from $[n]$ and $\Theta$ uniformly from $\F_2^n$. The target is $h_{J,\Theta}$. The learner's seed is independent of these variables and of the instance stream.

For the canonical analysis, define the learner's predictions arbitrarily on any history that no target can realize. This completes its decision rule on every record without changing its behavior on any genuine realizable interaction.

\begin{lemma}[Canonical-record accounting]\label{lem:canonical}
Fix a seed and an instance sequence. Run the learner on the record with actual anchor feedback and negative feedback at every informative round. For each group $g$, form the affine system $E_g$ by adding $\langle\theta,x\rangle=1-b$ whenever the prediction on an informative instance $x$ is $(g,b)$. A prediction $\bot$ adds no equation. A target $h_{g,\theta}$ produces this exact record if and only if $\theta$ satisfies $E_g$. A consistent rank-$r_g$ system therefore retains prior target mass $2^{-r_g}/n$.

For a deterministic output $\hat h$, let $p_g$ denote the fraction of $V_n$ on which it outputs either $(g,0)$ or $(g,1)$. Define $\Gamma(\hat h)$ as the smallest group maximizing $p_g$.
If $\err_\cD(\hat h)\le\eps$ under target group $J$, then $\Gamma(\hat h)=J$.
\end{lemma}
\role{This identifies what the negative record means under the actual protocol, and reduces low risk to correct group identification in Theorem~\ref{thm:interval}.}
\begin{proof}
At an anchor, the correct label is always $\bot$, independent of the target. At an informative point, a prediction naming another group is negative automatically. A prediction $(g,b)$ is negative under target $h_{g,\theta}$ exactly when $\langle\theta,x\rangle\ne b$, equivalently $\langle\theta,x\rangle=1-b$. A prediction $\bot$ is negative for every target.

Induct over rounds. Before the first round, every target is consistent with the empty record and each system is empty. Suppose a candidate target has produced the canonical record through the preceding round. Its current prediction is then the canonical prediction, since the learner's seed and observed history are fixed. The preceding case analysis says that its next response agrees with the canonical response exactly when it satisfies the new equation, if one is imposed on its group. Equations for other groups do not restrict it. This proves both directions of the induction: parameters satisfying the collected system produce the record, and parameters producing the record satisfy the system.

For a consistent system in group $g$, write its coefficient matrix as $M_g$, its right-hand side as $b_g$, and choose one solution $\theta_g$. Its full solution set is
\[
\{\theta:M_g\theta=b_g\}=\theta_g+\ker M_g.
\]
Indeed, subtracting $\theta_g$ from a solution gives a kernel vector, and adding any kernel vector to $\theta_g$ gives another solution. Rank-nullity gives $\dim\ker M_g=n-r_g$. Choosing its free binary coordinates yields exactly $2^{n-r_g}$ solutions. Each pair $(g,\theta)$ has prior mass $1/(n2^n)$, so their total mass is $2^{-r_g}/n$. This is unnormalized prior mass, not a posterior probability conditioned on observing the canonical record.

Let $p_J$ be the fraction of informative points at which $\hat h$ names group $J$. All other informative points are classification errors, irrespective of the predicted bit, so risk is at least $4\eps(1-p_J)$. Risk at most $\eps$ thus implies $p_J\ge3/4$. The shares of distinct groups are disjoint and sum to at most one. Hence $J$ is the unique maximizing group. This argument permits the output to use $\bot$ on informative points and does not require it to belong to the class. A randomized output is handled by including its randomness in the seed.
\end{proof}

\Needspace{12\baselineskip}
\begin{lemma}[An adaptive low-rank event]\label{lem:lowrank}
Let $n\ge16$ be divisible by four and $1\le\ell\le n/4$ be an integer. Put $M_0=n\ell/4$. If the learner uses $m\le n\ell/(64\eps)$ rounds, then with probability at least $11/16$ the following both hold on the canonical execution:
\begin{enumerate}[label=(\roman*)]
\item the number $N$ of informative observations is at most $M_0$;
\item whenever a consistent group of rank less than $\ell$ is queried at an informative observation, its new row is independent of its previous rows.
\end{enumerate}
The probability is over the instance sequence, for each fixed seed.
\end{lemma}
\role{This is the step that protects the lower bound against choosing a group after observing the current instance. It supports the survivor count of Lemma~\ref{lem:survivors}.}
\begin{proof}
The number of informative observations satisfies $N\sim\operatorname{Bin}(m,4\eps)$. Therefore
$\E N=4\eps m\le n\ell/16=M_0/4$,
and Markov's inequality gives $\Pr[N>M_0]\le1/4$.

To make the required independence explicit, generate the instance stream using independent Bernoulli variables of parameter $4\eps$ to specify the informative positions, together with an independent sequence of uniform vectors from $V_n$ to fill those positions. Fill every other position with $a$. This has exactly the required i.i.d. law. In particular, observing the number of intervening anchor rounds does not reveal the next informative vector.

Consider informative observations in their order of arrival. Immediately before one of them, the canonical history fixes the row space of every currently consistent group. A binary row space of rank $r$ has $2^r$ elements. A group of rank at most $\ell-1$ therefore has at most $2^{\ell-1}$ vectors in its row space. There are at most $n$ such groups, so their union has at most $n2^{\ell-1}$ elements. Counting the zero vector and counting shared vectors more than once only increases this bound. The arriving informative vector is uniform on $V_n$ independently of the preceding history. Thus its conditional probability of lying in this union is at most
\[
\frac{n2^{\ell-1}}{2^n-1}.
\]
This event is defined before a queried group is selected. If the new vector lies outside the union, it is independent of the rows of every currently consistent low-rank group, so it is independent for whichever such group the learner chooses after seeing the vector.

Extend the independent instance stream and let the canonical learner predict $\bot$ after its prescribed horizon. This freezes all systems and makes later informative observations well-defined without changing the actual execution. For each positive integer $k$, let $F_k$ be the event that the $k$th informative vector lies in the union just described for the history before that vector. Averaging the conditional bound over that history gives the same upper bound on $\Pr[F_k]$. No independence between the events $F_k$ is needed. A union bound over the first $M_0$ informative observations gives failure probability at most
\[
\beta:=\frac{M_0 n2^{\ell-1}}{2^n-1}
=\frac{n^2\ell2^{\ell-3}}{2^n-1}
\le n^3 2^{-3n/4-4}\le\frac1{16}.
\]
The first inequality uses $\ell\le n/4$ and $2^n-1\ge2^{n-1}$. For the second, the expression equals $1/16$ at $n=16$, and its ratio at $n+4$ to its value at $n$ is $(1+4/n)^3/8<1$.

Combining $\Pr[N>M_0]\le1/4$ with $\Pr[\bigcup_{k=1}^{M_0}F_k]\le1/16$ gives probability at least $11/16$ for their joint complement. We do not condition on $N\le M_0$ when estimating any $F_k$. On the joint complement, every actual informative round is among those controlled by the union bound, proving the lemma.
\end{proof}

\begin{lemma}[Surviving groups force failure]\label{lem:survivors}
Under the assumptions of Lemma~\ref{lem:lowrank}, let $\hat h$ be the learner's final predictor and use the group estimator $\Gamma$ from Lemma~\ref{lem:canonical}. Its prior-average probability $\Pr[\Gamma(\hat h)\ne J]$ is at least
\[
\frac{11}{16}\left(\frac34-\frac1n\right)2^{-(\ell-1)}
\ge\frac{121}{128}2^{-\ell}.
\]
\end{lemma}
\role{This converts the rank event into a statistical lower bound. It uses the exact target masses from Lemma~\ref{lem:canonical}, not an independence assumption between groups.}
\begin{proof}
Fix a seed and an instance sequence in the event of Lemma~\ref{lem:lowrank}.

\textit{A new independent row preserves consistency.}
Write a consistent current system as $M\theta=b$, put $r=\operatorname{rank}(M)$, and choose a solution $\theta_0$. Suppose a new row $x$ is outside the row space of $M$. Appending it raises the matrix rank from $r$ to $r+1$. By rank-nullity,
\[
\dim\ker M=n-r,\qquad
\dim\{v:Mv=0,\ \langle v,x\rangle=0\}=n-r-1.
\]
The second space is a proper subspace of the first. Hence there is a vector $v\in\ker M$ with $\langle v,x\rangle\ne0$, which over $\F_2$ means $\langle v,x\rangle=1$. Both $\theta_0$ and $\theta_0+v$ solve the old system, and their values at $x$ are opposite. One of them satisfies either prescribed binary right-hand side for the new equation. Thus the extended system remains consistent and its rank increases by one.

\textit{Count the groups that can leave the low-rank regime.}
By Lemma~\ref{lem:lowrank}, every query to a consistent group of rank less than $\ell$ adds such an independent row. Thus a group cannot become inconsistent before accumulating $\ell$ independent queried rows. Reaching rank $\ell$ also requires at least $\ell$ queries to that group.

Let $C_g$ be the number of informative queries addressed to group $g$, and let $\mathcal G_{\mathrm{spent}}$ be the groups that have reached rank $\ell$ or become inconsistent. Each informative query addresses at most one group, so
\[
\ell|\mathcal G_{\mathrm{spent}}|
\le\sum_{g\in\mathcal G_{\mathrm{spent}}}C_g
\le\sum_{g=1}^n C_g\le N\le n\ell/4.
\]
Thus at most $n/4$ groups belong to $\mathcal G_{\mathrm{spent}}$, and at least $3n/4$ groups remain consistent of rank at most $\ell-1$.

\textit{Convert surviving target mass into failure probability.}
Let $\hat h_0$ be the output on the canonical record and let $\widehat J_0=\Gamma(\hat h_0)$. All parameters satisfying their group's system produce that same record, by Lemma~\ref{lem:canonical}. For each surviving group $g\ne\widehat J_0$, this output misidentifies the group. Its surviving parameters have prior mass at least $2^{-(\ell-1)}/n$. There are at least $3n/4-1$ such groups. Conditional on this sequence and seed, the target-prior failure probability is thus at least
$ (3/4-1/n)2^{-(\ell-1)}$.
The good instance event is determined by the seed and the instance sequence through the canonical execution, without using the random target. Thus the target prior remains uniform when these are fixed. On this event the displayed lower bound holds; outside it, failure probability is nonnegative. The law of total probability therefore permits multiplication by the event's probability, at least $11/16$. This proves the bound for each fixed seed, and averaging over the seed proves the randomized statement. Finally, $n\ge16$ gives
$3/4-1/n\ge11/16$, so the product is at least
$(11/16)^2 2^{-(\ell-1)}=(121/128)2^{-\ell}$.
\end{proof}

\subsection{The Interval Lower Bound}\label{app:interval}
\role{This completes the direct guarantee for the lower-bound module in Section~\ref{sec:am-lower}. Its dependencies are Lemmas~\ref{lem:canonical}, \ref{lem:lowrank}, and~\ref{lem:survivors}.}
\begin{proof}\textbf{of Theorem~\ref{thm:interval}}
Write $\kappa=\log_2(1/\delta)$. The specified confidence range is
$2\le\kappa\le n/4+2$.
The choice $\ell=\min\{n/4,\lfloor\kappa\rfloor-1\}$ is an integer satisfying
\[
1\le\ell\le n/4,\qquad \ell\le\kappa-1,
\qquad \ell\ge\kappa/3.
\]
For the last inequality, if $2\le\kappa<3$ then $\ell=1$. If $\kappa\ge3$ and there is no truncation, $\ell\ge\kappa-2\ge\kappa/3$. If truncation occurs, $\ell=n/4$ and $\kappa\le n/4+2\le3n/4$. The inequality $\ell\le\kappa-1$ follows directly in the untruncated case; truncation requires $\lfloor\kappa\rfloor-1\ge n/4$, which gives it in the other case.

For any integer $m\le n\ell/(64\eps)$, Lemma~\ref{lem:survivors} and $\ell\le\kappa-1$ yield a group-error probability at least
\[
\frac{121}{128}2^{-\ell}
\ge\frac{121}{64}2^{-\kappa}
=\frac{121}{64}\delta>\delta.
\]
By Lemma~\ref{lem:canonical}, group error entails risk greater than $\eps$. This is an average over finitely many realizable targets, so one fixed target has failure probability greater than $\delta$. Thus no integer at most $\lfloor n\ell/(64\eps)\rfloor$ is sufficient. Finally,
\[
\left\lfloor\frac{n\ell}{64\eps}\right\rfloor+1
>\frac{n\ell}{64\eps}
\ge\frac{n\kappa}{192\eps}
=\frac{nL}{192\ln2\,\eps}.
\]
At the endpoint $\delta=2^{-n/4-2}$, $\ell=n/4$, so the same proof gives
$\mB_{\AM}\ge\lfloor n^2/(256\eps)\rfloor+1$.
\end{proof}

\Needspace{9\baselineskip}
\subsection{The Two Learning Procedures}\label{app:am-upper}
\begin{lemma}[Rank exhaustion]\label{lem:saturation}
The rank-exhaustion procedure of Section~\ref{sec:am-algorithms} has nested error sets and at most $n^2+n-1$ negative responses on every realizable sequence.

If the target group is known, the same within-group predictor, initialized with an empty affine system in that group, has nested error sets and at most $n$ negative responses on every realizable sequence.
\end{lemma}
\role{This supplies the deterministic hypotheses of Lemma~\ref{lem:monotone} for the second guarantee in Proposition~\ref{prop:am-upper}. Its within-group part is also used after positive search.}
\begin{proof}
Fix any target $h_{J,\Theta}$ and any instance sequence. Here $J$ and $\Theta$ are fixed parameters, not draws from the lower-bound prior. We verify the update rule, count its negative responses, and then check that its error sets are nested.

\textit{Specify the predictor and preserve a consistent system.}
Start at group $g=1$ with an empty affine system and no group confirmation. At the anchor predict $\bot$ and do not update. Write the current consistent system as $M_g\theta=b_g$ and its row space as $W$. At an informative $x\in W$, choose coefficients $c$ with $x=M_g^{\mathsf T}c$. Every solution then has
\[
\langle\theta,x\rangle=c^{\mathsf T}M_g\theta=c^{\mathsf T}b_g.
\]
The last value is independent of the choice of coefficients: any two representations of $x$ evaluate equally on a solution of the consistent system. Predict this forced bit with group label $g$. For $x\notin W$, predict $(g,0)$.

On an informative positive, mark the group confirmed and leave the system unchanged. On a negative with $x\notin W$, add $\langle\theta,x\rangle=1$. The maintained rows are independent by construction. If there are $r$ old rows, appending $x$ gives $r+1$ independent rows. The linear map
\[
\theta\longmapsto (M_g\theta,\langle\theta,x\rangle)
\]
has image of dimension $r+1$, equal to the dimension of its codomain $\F_2^{r+1}$. It is therefore onto and takes the value $(b_g,1)$ for some $\theta$. This proves consistency of the new system, whose rank has increased by one.

On a negative with $x\in W$, discard the group, advance to the next group, and reset the system. Every parameter satisfying the old system gives the predicted forced bit at $x$, so this negative excludes all of them. This case cannot occur in a confirmed group, as verified next.

\textit{Bound the total number of negatives.}
Suppose the current group is the target group $J$. Every independent negative reveals the true equation $\langle\Theta,x\rangle=1$. By induction, $\Theta$ satisfies every maintained equation. A forced prediction at $x\in W$ is therefore correct and cannot eliminate this group. An informative positive in any group also identifies it as $J$, since the group is part of the observed label. Thus a confirmed group is never eliminated. A wrong group produces a negative on every informative round. It can accumulate at most $n$ independent rows before the next informative negative eliminates it. The true group receives at most $n$ independent negatives. There are at most $n-1$ wrong groups before it, giving
\[
(n-1)(n+1)+n=n^2+n-1
\]
negatives. The learner cannot run out of groups, because the true group cannot be discarded.

\textit{Verify monotonicity of the error sets.}
A wrong current group misclassifies every point of $V_n$ and predicts the anchor correctly. Switching to another wrong group leaves this error set unchanged; switching to the true group can only reduce it. While in the true group, its current error set is exactly
\[
\{x\in V_n\setminus W:\langle\Theta,x\rangle=1\}.
\]
Forced predictions inside $W$ are true, and the default outside it is zero. As independent negatives are added, $W$ increases and all equations remain true. This displayed set can only shrink. Positive responses do not change the system or predictor. These facts establish the required nesting throughout the entire execution.

\textit{Start directly in the known target group.}
For the second assertion, initialize the empty system in group $J$. The preceding argument shows that every maintained equation is satisfied by $\Theta$, so this group is never eliminated. Each negative increases the rank by one, and the rank is at most $n$. Hence there are at most $n$ negatives. The same displayed error set applies from initialization onward and is nested throughout this within-group execution.
\end{proof}

\begin{lemma}[Positive search followed by within-group learning]\label{lem:positive-search}
For $n\ge16$, $0<\eps\le1/16$, and $0<\delta\le1/4$, the positive-search procedure is PAC with at most $10nL/\eps$ rounds.
\end{lemma}
\role{This certifies the first algorithm in Proposition~\ref{prop:am-upper}. The search is analyzed directly; only its second phase uses Lemmas~\ref{lem:saturation} and~\ref{lem:monotone}.}
\begin{proof}
\textit{Reserve separate search and learning phases.}
Put $\eta=\delta/3$, $v=\ln(1/\eta)=L+\ln3$, and
\[
t=\lceil\log_2(3/\delta)\rceil,\qquad N_0=nt,
\qquad m_1=\left\lceil\frac{\max\{2N_0,8v\}}{\eps}\right\rceil,
\qquad m_2=\left\lceil\frac{\max\{2n,8v\}}{\eps}\right\rceil.
\]
Here $t$ is the number of complete passes through the groups, $N_0$ is the corresponding number of informative observations, and $m_1,m_2$ are the raw-round budgets for the two phases.
Reserve exactly $m_1$ rounds for search. On informative observations, cycle through groups in a fixed order and predict an independent uniformly random bit. Stop the search when an informative positive occurs. If this happens before the reserved phase ends, retain the identified group and ignore the remaining observations in that phase. If no positive has occurred after each group has been tested $t$ times, the remaining search observations can likewise be ignored. On every ignored round predict $\bot$ and discard the feedback. Predict $\bot$ on anchor rounds throughout.

\textit{Control failure to identify the group.}
Let $p=\cD(V_n)$. If $p\le\eps$, the prescribed fallback $h_{1,0}$ has risk at most $\eps$. The predictor returned after successful search also predicts $\bot$ at the anchor. Thus either branch has risk at most $p\le\eps$, whether or not its final predictor is proper. Suppose instead that $p>\eps$. The number of informative points among the first $m_1$ observations is binomial with mean $\mu=m_1p\ge m_1\eps$. Since $\mu\ge2N_0$ and $\mu\ge8v$, the event of fewer than $N_0$ informative points is contained in the event of at most $\mu/2$ such points. Equation~\eqref{eq:binomial-half-tail} bounds its probability by $e^{-\mu/8}\le\eta$.

On a hypothetical search that continues regardless of positives, the true group is tested $t$ times in its first $N_0$ informative observations. Each independent fair bit matches the target bit with probability exactly $1/2$, conditional on the instances and the previous bits. The probability of missing on all these tests is $2^{-t}\le\eta$. A wrong group can never give an informative positive. The actual and hypothetical searches agree until the first positive, so a positive in these tests also identifies the group in the actual procedure.

For clarity, let $\mathcal C$ be the event that the reserved search block contains at least $N_0$ informative observations, and let $\mathcal S$ be successful identification. Conditional on any instance sequence in $\mathcal C$, all the required fair-bit tests are available, and their simultaneous failure has probability $2^{-t}$. Consequently,
\[
\Pr[\mathcal S^c]
\le\Pr[\mathcal C^c]+\Pr[\mathcal S^c\cap\mathcal C]
\le\eta+2^{-t}\le2\eta.
\]

\textit{Learn the identified group on fresh observations.}
If identification succeeds, run the within-group algorithm from an empty system for the next $m_2$ fresh rounds. Return the predictor maintained after the last of these rounds and its update. By the known-group assertion of Lemma~\ref{lem:saturation}, its negatives are bounded by $n$ and its error sets are nested. The second phase begins at a fixed time, so these rounds remain i.i.d. and independent of the identification event. Applying Lemma~\ref{lem:monotone} to this final predictor bounds its conditional failure probability by $\eta$. The total is at most $3\eta=\delta$.

\textit{Verify the stated numerical budget.}
Since $L\ge\ln4$,
\[
v=L+\ln3\le\frac95L,\qquad
t\le\frac{L+\ln6}{\ln2}\le\frac{7}{3\ln2}L.
\]
The second inequality uses $\ln6/\ln4<4/3$. Also $2nt\ge2nv/\ln2\ge8v$ for $n\ge16$, so the first maximum equals $2nt$. Using $\max\{2n,8v\}\le2n+8v$ and the two ceiling errors gives $\eps(m_1+m_2)\le2nt+2n+8v+2\eps$. Dividing by $nL$ yields
\[
\frac{\eps(m_1+m_2)}{nL}
\le\frac{14}{3\ln2}+\frac{2}{\ln4}
  +\frac{72}{5n}+\frac{2\eps}{nL}
< 6.74+1.45+0.9+0.006<10.
\]
The penultimate bound uses $n\ge16$, $\eps\le1/16$, and $L\ge\ln4$. This proves the stated constant.
\end{proof}

\role{The direct algorithmic proposition in the main text combines the preceding two analyses with the shared monotone conversion.}
\begin{proof}\textbf{of Proposition~\ref{prop:am-upper}}
Lemma~\ref{lem:positive-search} gives the search bound. Run rank exhaustion for
\[
m=\left\lceil\frac{\max\{2(n^2+n-1),8L\}}{\eps}\right\rceil
\]
rounds and return the predictor maintained after round $m$ and its update. Lemma~\ref{lem:saturation} gives nested error sets and at most $n^2+n-1$ negatives. Lemma~\ref{lem:monotone} therefore bounds the failure probability of this final predictor by $\delta$.
Since a ceiling adds at most one, its numerator is at most
$2(n^2+n-1)+8L+\eps\le9(n^2+L)$
for $n\ge16$ and $\eps\le1/16$. Thus $m\le9(n^2+L)/\eps$. Both procedures apply to every distribution, unlike the distribution chosen solely for the lower bound.
\end{proof}

\subsection{Combining the Two Confidence Regimes}\label{app:profile}
\role{Theorem~\ref{thm:profile} combines Theorem~\ref{thm:interval}, Proposition~\ref{prop:am-upper}, and the ordinary confidence lower bound of Lemma~\ref{lem:confidence-lower}.}
\begin{proof}\textbf{of Theorem~\ref{thm:profile}}
The upper bound follows by choosing the smaller of the two algorithmic budgets. If $L\le n$, then
$10nL\le10(nL+L)=10\Psi_n(L)$.
If $L>n$, then $9(n^2+L)\le10\Psi_n(L)$. Divide by $\eps$.

For the lower bound, first suppose $\delta\ge\delta_n:=2^{-n/4-2}$. In this range
$L\le(n/4+2)\ln2<n$,
and $\Psi_n(L)=(n+1)L\le(17/16)nL$ because $n\ge16$. Theorem~\ref{thm:interval} gives
\[
\mB_{\AM}>\frac{nL}{192\ln2\,\eps}
\ge\frac{\Psi_n(L)}{512\eps}.
\]
For $\delta<\delta_n$, a learner meeting failure probability $\delta$ also meets the weaker requirement $\delta_n$. The endpoint of Theorem~\ref{thm:interval} therefore implies $\mB_{\AM}\ge n^2/(256\eps)$. Lemma~\ref{lem:am-dimensions} gives $A>0$, so Lemma~\ref{lem:confidence-lower} also implies $\mB_{\AM}\ge L/(8\eps)$. Hence
\[
\mB_{\AM}\ge\max\left\{\frac{n^2}{256\eps},\frac{L}{8\eps}\right\}
\ge\frac{n^2+L}{512\eps}
\ge\frac{\Psi_n(L)}{512\eps}.
\]
The middle inequality is obtained by taking one half of the sum of the two lower bounds and reducing the coefficient of $L$. This proves the full displayed profile.

For the regime descriptions, when $\delta\ge\delta_n$ we have $\Psi_n(L)=(n+1)L=\Theta(nL)$. When $\delta\le\delta_n$, put $c_0=(\ln2)/4\in(0,1)$. Then $L\ge(n/4+2)\ln2\ge c_0n$, so $\min\{n,L\}\ge c_0n$. Consequently
\[
c_0(n^2+L)\le c_0n^2+L
\le n\min\{n,L\}+L\le n^2+L.
\]
This proves the claimed $\Theta((n^2+L)/\eps)$ order throughout the second regime, including when $L$ lies between the stated threshold and $n$.
\end{proof}

\subsection{Impossibility and the Same-Parameter Comparison}\label{app:consequences}
\role{This proof supports the conclusions drawn in Section~\ref{sec:consequences}. Its class-specific inputs are Corollary~\ref{cor:rare}, Lemma~\ref{lem:am-dimensions}, and Theorem~\ref{thm:profile}.}
\begin{proof}\textbf{of Corollary~\ref{cor:separation}}
Fix $\eps=1/16$ and take $\delta_n=2^{-n/4-2}$. Then $L_n=(n/4+2)\ln2=\Theta(n)$. The multiplexer has $B=2n-1$, whereas Theorem~\ref{thm:profile} gives complexity $\Theta(n^2/\eps)$. A universal bound $C(B+L_n)/\eps$ would be only $O(n/\eps)$, a contradiction as admissible $n$ tends to infinity. More generally, $f(\AM)=o(n^2)$ gives $f(\AM)+L_n=o(n^2)$, which yields the same contradiction.

For the class $\cG_{2n}$, Corollary~\ref{cor:rare} gives the same triple $(A,B,R)$ and rate $\Theta((2n+L_n)/\eps)=\Theta(n/\eps)$. Write
\[
m_1=\mB_{\cG_{2n}}(\eps,\delta_n),\qquad
m_2=\mB_{\AM}(\eps,\delta_n).
\]
The upper bound for $m_1$ and the lower bound for $m_2$ have universal constants. Their ratio therefore satisfies $m_2/m_1\ge c_1n$ for a universal $c_1>0$.

Any positive-valued function of $(A,B,R,\eps,\delta)$ takes one value, say $z$, at this point. Its multiplicative error relative to a positive complexity $m$ is $\max\{z/m,m/z\}$. If this error is at most $r$ on both classes, then $z\le rm_1$ and $m_2\le rz$. Combining these two inequalities gives $m_2\le r^2m_1$, hence $r\ge\sqrt{c_1n}$. The larger of the two errors must therefore meet the claimed lower bound.

For an explicit polylogarithmic formulation, set
\[
\Lambda=\ln\left(2+A+R+\frac1\eps+\ln\frac1\delta\right).
\]
At the common parameter point, $\Lambda=\Theta(\ln n)$. If fixed constants $C,p$ and a function $F(A,B,R,\eps,\delta)$ satisfied
$\mB_\cH/(C\Lambda^p)\le F\le C\Lambda^p\mB_\cH$
for every class, the ratio $m_2/m_1$ would be at most $C^2\Lambda^{2p}=O((\ln n)^{2p})$. This contradicts its linear growth in $n$. The logarithmic variables in this statement are explicit; in particular the claim is not about factors polynomial in $\ln(1/\delta)$ itself.
\end{proof}

\section{Proofs for the General Algorithms}\label{app:upper}
We first derive the full-information list-learning input from existing results, then prove the two bandit reductions that use it. The first reduction collects correctly labeled examples and allocates resources across list sizes. The second obtains one positive response before choosing the class to which the first reduction is applied.

\subsection{The List-Learning Input}
For an integer $k\ge1$, let $\DS_k(\cH)$ be the supremum of integers $d\ge1$ for which there are a sequence $S\in\cX^d$ and a finite nonempty set $V\subseteq\cH|_S$ such that every $v\in V$ has at least $k$ distinct $i$-neighbors in $V$ for every $i\in[d]$. The value is zero if no such witness exists. Definition~\ref{def:dimensions} gives
\begin{equation}\label{eq:list-dimension}
k\DS_k(\cH)\le B.
\end{equation}
Indeed, each finite witness on $d$ coordinates with multiplicity $k$ has total value $kd$. This argument also shows that $\DS_k$ is finite whenever $B$ is finite.

For a distribution $Q$ on $\cX\times\cY$, write
$\mathcal L_Q(\mu)=\Pr_{(x,y)\sim Q}[y\notin\mu(x)]$ for list coverage error. A distribution $Q$ is realizable if there is a single $h^\star\in\cH$ such that $\Pr_{(x,y)\sim Q}[y=h^\star(x)]=1$.

\begin{proposition}[List-PAC input]\label{prop:list-blackbox}
Let $R=\wid(\cH)\ge2$ and let $\ell\ge1$ be an integer. There is a deterministic full-information list learner, depending only on $\cH$ and $\ell$, such that, for every integer $r\ge1$ and every $\alpha\in(0,1)$, on $r$ i.i.d. examples from any realizable $Q$, its output has lists of size at most $2\ell-1$ and, with probability at least $1-\alpha$,
\[
\mathcal L_Q(\mu)\le
\frac{100}{r}\left(\DS_{\lceil\ell/2\rceil}(\cH)+
\ln\frac2\alpha\right).
\]
\end{proposition}
\role{This supplies the labeled-sample requirement of the list-update step in Section~\ref{sec:cascade}. We combine the density bound and leave-one-out construction of \citet[Theorem~1 and Appendix~A.1.1]{Pabbaraju26} with the suffix-risk conversion of \citet[Theorem~2.1]{AdenAliEtAl23}. The strict-majority aggregation is also used in \citet[Lemma~4.6]{HMMSv3}. From these results, we derive the list-size and coverage guarantees required by the cascade, with explicit constants. Our cascade permits output lists of size at most $2\ell-1$ rather than $\ell$. This allows strict-majority aggregation with only a factor-two increase in coverage error.}
\begin{proof}\textit{Handle trivial cases and reduce the label width.}
If $2\ell-1\ge R$, return the realized label set at each point. Its size is at most $R$ and its coverage error is zero. It therefore suffices to handle $2\ell-1<R$.
Write $d=\DS_{\lceil\ell/2\rceil}(\cH)$. If $d=\infty$, the asserted error bound is vacuous and any list predictor of the allowed size suffices. If $r<4$, its right-hand side is at least $100\ln(2)/r>1$, so again any fixed list predictor suffices. Henceforth $d<\infty$ and $r\ge4$.

For each $x$, let $\cY_x=\{h(x):h\in\cH\}$ and fix an injection $\phi_x:\cY_x\to[R]$. Define a transformed class by
$\widetilde{\cH}=\{\widetilde h:h\in\cH\}$, where $\widetilde h(x)=\phi_x(h(x))$.
On every finite sequence, the coordinatewise injections preserve equality and inequality of labels and hence all neighbor relations. They induce a bijection on restriction vectors. Witnesses in either representation therefore pull back and push forward with exactly the same multiplicities. In particular all $\DS_k$ dimensions are unchanged.

Transform each observed pair $(x,y)$ to $(x,\phi_x(y))$, and denote its distribution by $\widetilde Q$. These examples remain i.i.d. and realizable by $\widetilde{\cH}$.

\textit{Obtain a symmetric predictor with bounded leave-one-out error.}
Put $D=\DS_\ell(\widetilde{\cH})$. A witness with at least $\ell$ neighbors in each coordinate also has at least $\lceil\ell/2\rceil$ neighbors there. The preserved dimensions therefore give $D\le d$.
Theorem~1 and the derivation in Appendix~A.1.1 of \citet{Pabbaraju26} give a deterministic one-inclusion list predictor $\nu(\cdot;S)$, with lists of size at most $\ell$, that is invariant under permutations of its training sample $S$. For every realizable sample $S=((x_i,y_i))_{i=1}^m$, it satisfies
\[
\sum_{i=1}^m\ind{y_i\notin\nu(x_i;S^{-i})}\le D,
\]
where $S^{-i}$ is $S$ with example $i$ removed. The bound follows by orienting the one-inclusion hypergraph with maximum $\ell$-outdegree at most the ceiling of its maximum $\ell$-density, which the cited density theorem bounds by $D$.

To define $\nu$ on every training sample, set $\nu(x;S)=\varnothing$ for every $x$ whenever $S$ is not realizable by $\widetilde{\cH}$. This branch is invariant under permutations and is never used on a prefix or leave-one-out subsample of a realizable sample. For realizable $S$, choose the orientation by a fixed rule for each finite set of instances, independently of their order or of the held-out coordinate. If a test instance already occurs in the training sample, return the singleton containing its observed label. In a realizable sample, every leave-one-out test at a repeated instance is then correct. At the remaining instances, the errors are bounded by the outdegree of the target vertex in the same hypergraph on the set of distinct instances. This explains why repetitions do not invalidate either the bound or permutation invariance. All selections use the measurability convention of Section~\ref{sec:model}.

\textit{Convert the leave-one-out bound into a suffix-risk bound.}
Let $q=4\lfloor r/4\rfloor$, so $r/2\le q\le r$. Keep the first $q$ transformed examples and call this sample $S$. The discarded observations, at most three, are selected by their positions rather than their values. Thus $S\sim\widetilde Q^q$.
For $t=q/4,\ldots,q-1$, define $\nu_t=\nu(\cdot;S_{\le t})$, where $S_{\le t}$ is the first $t$ examples. There are $N=3q/4$ such predictors.

Theorem~2.1 of \citet{AdenAliEtAl23} requires permutation invariance and a leave-one-out upper bound that is nondecreasing in sample size. The constant sequence $M_m=D$ satisfies the latter requirement. Apply the theorem with the bounded loss $\ind{y\notin V}$ for a predicted list $V$ and true label $y$. Formally, its common prediction and label space can be taken to be the disjoint union of $[R]$ and its subsets of size at most $\ell$. True labels lie in the first part and predictions in the second. Define the loss as above on these pairs and arbitrarily in $[0,1]$ elsewhere. This puts list prediction in exactly the theorem's bounded-loss setting. Since $q$ is divisible by four, that theorem gives, with probability at least $1-\alpha$,
\[
\frac1N\sum_{t=q/4}^{q-1}\mathcal L_{\widetilde Q}(\nu_t)
\le\frac{4.82}{q}\left(D+\ln\frac2\alpha\right).
\]
The logarithm in the cited theorem is natural. The predictors depend on the sample, but neither their construction nor the choice of $q$ uses $\alpha$.

\textit{Aggregate without exceeding the required list size.}
For each instance $x$, return every label appearing in strictly more than half of the $N$ lists:
\[
\widetilde\mu(x)=\left\{y\in[R]:
\sum_{t=q/4}^{q-1}\ind{y\in\nu_t(x)}>N/2\right\}.
\]
There are at most $N\ell$ label occurrences across all these lists. If $\widetilde\mu(x)$ contained at least $2\ell$ labels, its labels alone would account for strictly more than $2\ell(N/2)=N\ell$ occurrences. Hence $|\widetilde\mu(x)|\le2\ell-1$.
If a true label $y$ is missing from $\widetilde\mu(x)$, it is missing from at least half of the constituent lists. Consequently,
\[
\ind{y\notin\widetilde\mu(x)}
\le\frac2N\sum_{t=q/4}^{q-1}\ind{y\notin\nu_t(x)}.
\]
Taking expectation under $\widetilde Q$ and using the suffix-risk event gives
\begin{align*}
\mathcal L_{\widetilde Q}(\widetilde\mu)
&\le\frac{9.64}{q}\left(D+\ln\frac2\alpha\right)\\
&\le\frac{19.28}{r}\left(d+\ln\frac2\alpha\right)
\le\frac{100}{r}\left(d+\ln\frac2\alpha\right).
\end{align*}
The second inequality uses $q\ge r/2$ and $D\le d$.

At a test point $x$, remove from the returned list any index outside the image of $\phi_x$ and map the remaining indices back to $\cY_x$. This never removes a possible correct transformed label. It does not increase list size or coverage error. Thus it produces the claimed list predictor in the original label space. Empty or shorter lists may subsequently be padded with labels, which also cannot increase error. No inverse is assigned to an index outside the image.
\end{proof}

\subsection{The Positive-Example Collection Module}
\begin{lemma}[Fixed-width collection]\label{lem:harvest}
Fix a list predictor $\mu$ of size exactly $s$ at every point, with coverage error $e\le1/2$. On fresh bandit rounds, predict uniformly from $\mu(x)$ and retain the labeled pair precisely when feedback is positive. The retained pairs form an i.i.d. sequence from
$Q=P_{\cD,h^\star}\mid\{h^\star(x)\in\mu(x)\}$,
where $P_{\cD,h^\star}$ is the law of $(x,h^\star(x))$ for $x\sim\cD$.
For an integer $r\ge1$, $0<\alpha<1$, and $u=1+\ln(1/\alpha)$, the first $r$ retained examples are obtained within
$\lceil16s(r+u)\rceil$ raw rounds except with probability at most $\alpha$.
The retained pairs are independent of their retention times. In particular, for every fixed raw-round budget $M\ge r$, conditional on collecting at least $r$ pairs within $M$ rounds, the first $r$ retained pairs still have joint law $Q^r$.
\end{lemma}
\role{This verifies both the data distribution and the raw-round budget of each stage in Section~\ref{sec:cascade}. The list predictor is fixed conditional on the preceding stages, which is sufficient for each application.}
\begin{proof}
\textit{Identify the law of one retained example.}
For a fresh pair $(x,y)$, the probability of retention conditional on the pair is
\[
\Pr[\text{retain}\mid x,y]=\frac1s\ind{y\in\mu(x)}.
\]
For any measurable set $G$ of labeled pairs, the probability that a single round is retained and belongs to $G$ is
$P_{\cD,h^\star}(G\cap\{y\in\mu(x)\})/s$.
Dividing by the retention probability $p=(1-e)/s$ gives exactly $Q(G)$. Since the list is fixed, independent raw rounds and independent prediction draws also make the retention indicators independent, each with success probability $p$.

\textit{Check independence, including after timely collection.}
We must verify not only the marginal law but also that stopping at retained observations does not introduce dependence. If $p=1$, every round is retained and the assertion follows directly from the original i.i.d. sequence. For $0<p<1$, consider an infinite continuation of the raw rounds. Let $\tau_j$ be the time of the $j$th retention, with $\tau_0=0$. Define the gap $D_j=\tau_j-\tau_{j-1}$ and the retained labeled pair $Z_j=(x_{\tau_j},h^\star(x_{\tau_j}))$. The gaps are finite almost surely because $p>0$. For positive integers $k_1,\ldots,k_r$ and measurable sets $G_1,\ldots,G_r$ of labeled pairs, independence of the raw rounds gives
\[
\Pr[D_j=k_j,\ Z_j\in G_j\text{ for all }j\le r]
=\prod_{j=1}^r(1-p)^{k_j-1}p\,Q(G_j).
\]
Each factor accounts for $k_j-1$ rejections followed by one retention whose value belongs to $G_j$. Summing over the positive integers $k_j$ gives $\prod_{j=1}^rQ(G_j)$. Taking every $G_j$ to be the entire labeled space gives the joint distribution of the gaps. The same factorization therefore proves both that the retained pairs are i.i.d. with law $Q$ and that they are independent of the gaps. In particular, conditioning on collection success $\tau_r\le M$ for a fixed raw budget $M$ does not change the joint law of the first $r$ retained pairs.

\textit{Bound the raw-round budget.}
Since $e\le1/2$, $p\ge1/(2s)$. For $M=\lceil16s(r+u)\rceil$ raw rounds, the number retained is binomial with mean $Mp\ge8(r+u)$. In particular $r\le Mp/2$, and \eqref{eq:binomial-half-tail} gives
\[
\Pr[\text{fewer than }r\text{ retained}]
\le e^{-Mp/8}\le e^{-(r+u)}\le\alpha.
\]
This proves the fixed budget. On collection success, the first $r$ retained pairs have the same i.i.d. law as in the infinite continuation, so they are valid input to the deterministic full-information learner. Thus neither selection of the retained pairs nor conditioning on timely collection changes that learner's input law.
\end{proof}

\subsection{Stage Composition and the Total Budget}\label{app:cascade}
\role{This proves Theorem~\ref{thm:cascade} for the complete list-shrinking module. The dependencies are the list input of Proposition~\ref{prop:list-blackbox}, the collection guarantee of Lemma~\ref{lem:harvest}, and the dimension estimate \eqref{eq:list-dimension}.}
\begin{proof}\textbf{of Theorem~\ref{thm:cascade}}
If $R=1$, every hypothesis equals the same function and zero samples suffice. Assume $R\ge2$. Use exactly the stage parameters in Section~\ref{sec:cascade}. They are determined by the known class and $(\eps,\delta)$, before observations are collected. They are all finite, since $B<\infty$ and \eqref{eq:list-dimension} bounds every $d_t$.
We check that the list sizes are feasible, that successful stages compose into the requested risk guarantee, and that their deterministic raw budgets sum to the asserted bound.

For definiteness, assign stage $t$ a fixed consecutive block of $m_t^{\mathrm{raw}}$ raw rounds. Throughout that block predict from the fixed preceding list $\mu_{t-1}$. Keep its first $n_t$ positive examples and discard later ones. Apply the deterministic list learner at the end of the block. If fewer than $n_t$ examples were retained, use the stated fallback and ignore the unused future budget. This implementation has the same stage budgets as the main-text procedure and makes each new stage's raw observations independent of the preceding history.

\textit{List sizes and dimension budgets.}
The integer $T=\lfloor\log_2R\rfloor$ is at least one. Since $1\le R/2^T<2$, we have $1/2\le R/2^{T+1}<1$, hence $\ell_T=1$ and $s_T=1$. The values $\ell_t$ and $s_t$ do not increase with $t$. For their largest value, $s_1=2\lceil R/4\rceil-1\le R/2+1\le R$ because $R\ge2$. Thus every padding size is at most $R\le K$, and there are enough distinct labels. The definitions give
\[
s_{t-1}\le4\ell_t,\qquad
s_{t-1}d_t\le4\ell_td_t
\le8\lceil\ell_t/2\rceil d_t\le8B.
\tag{\ensuremath{\ast}}\label{eq:stage-dimension}
\]
For $t=1$, the first bound is $R\le4\lceil R/4\rceil$. For $t\ge2$, it follows from
$\lceil 2z\rceil\le2\lceil z\rceil$ applied to $z=R/2^{t+1}$. The last inequality is \eqref{eq:list-dimension}.

\textit{A stage's statistical guarantee.}
Let $e_t=\Pr_{x\sim\cD}[h^\star(x)\notin\mu_t(x)]$. Initially $e_0=0$, because $\mu_0$ includes all realized labels. The positive allocation values $\eta_t$ sum to $\bar\eps$, directly from their definition. We will prove by induction that every successful prefix obeys $e_{t-1}\le\sum_{j<t}\eta_j\le\bar\eps\le1/2$.

Fix a history before stage $t$ satisfying this inductive claim. Its list $\mu_{t-1}$ is now fixed. Lemma~\ref{lem:harvest} bounds collection failure by $\alpha_t$. Conditional on collection success, its independence statement gives exactly $n_t$ i.i.d. examples from
\[
Q_t=P_{\cD,h^\star}\mid\{h^\star(x)\in\mu_{t-1}(x)\}.
\]
This distribution is well-defined because $1-e_{t-1}\ge1/2$. It is realizable by the same target $h^\star\in\cH$, so the external list learner's realizability assumption is satisfied.

For the retained distribution $Q_t$, Proposition~\ref{prop:list-blackbox} has numerator
$100(d_t+\ln(2/\alpha_t))$.
Since $u_t=1+\ln(1/\alpha_t)\ge\ln(2/\alpha_t)$, the choice of $n_t$ makes its coverage error at most $\eta_t$ except with conditional probability $\alpha_t$. Count a stage as failed if collection fails or if the returned list has conditional coverage error greater than $\eta_t$. For the fixed preceding history, the first event has probability at most $\alpha_t$. The second, intersected with collection success, has probability at most $\Pr[\text{collection succeeds}]\alpha_t\le\alpha_t$. Thus stage failure has conditional probability at most $2\alpha_t$. This estimate is uniform over every history satisfying the inductive claim.

Whenever neither failure occurs, split error according to membership in the previous list to get
\begin{align}
e_t
&=\Pr[h^\star(x)\notin\mu_t(x),\ h^\star(x)\in\mu_{t-1}(x)]\notag\\
&\quad+\Pr[h^\star(x)\notin\mu_t(x),\ h^\star(x)\notin\mu_{t-1}(x)]\notag\\
&\le(1-e_{t-1})\mathcal L_{Q_t}(\mu_t)+e_{t-1}
\le e_{t-1}+\eta_t.\label{eq:cascade-error}
\end{align}
The first term in the split concerns examples whose target label is covered by the preceding list. Their conditional distribution is exactly $Q_t$. The second term is bounded by the entire mass $e_{t-1}$ outside that coverage event, whether or not the new list recovers some of those labels. Thus the calculation does not require successive lists to be nested. Padding cannot increase coverage error. By induction, every successful prefix has
$e_t\le\sum_{j\le t}\eta_j\le\bar\eps$.
This establishes the required coverage hypothesis for the next stage. To make the probability calculation explicit, let $G_t$ be the event that the first $t$ stages succeed, with $G_0$ the whole probability space. Let $\mathcal F_{t-1}$ be the history before stage $t$, including prior algorithmic randomness. The events $G_{t-1}\setminus G_t$ are disjoint and specify the first failing stage. The conditional bound just proved gives
\[
\Pr[G_T^c]
=\sum_{t=1}^T\E\!\left[\ind{G_{t-1}}
\Pr[\text{stage }t\text{ fails}\mid\mathcal F_{t-1}]\right]
\le2\sum_{t=1}^T\alpha_t.
\]
There is no assumption of independence between stages. With the chosen allocation, this total is at most
\[
2\sum_{t=1}^T\alpha_t
\le2\delta\sum_{t=1}^{\infty}2^{-t-3}=\delta/4\le\delta.
\]
When all stages succeed, the final singleton list defines a predictor of risk at most $\bar\eps\le\eps$.

\textit{Summing raw rounds.}
Let $S_c=\sum_{t=1}^T\sqrt{c_t}$. Since $u_t\ge1$ and $s_{t-1}\ge1$, every $c_t$ is positive. The allocation $\eta_t=\bar\eps\sqrt{c_t}/S_c$ is therefore well-defined and sums to $\bar\eps$. To account for both ceilings, first use
\[
n_t\le\frac{100(d_t+u_t)}{\eta_t}+1.
\]
Substituting this into the raw budget gives
\begin{align*}
m_t^{\mathrm{raw}}
&\le16s_{t-1}(n_t+u_t)+1\\
&\le\frac{1600c_t}{\eta_t}+16s_{t-1}(1+u_t)+1\\
&\le\frac{1600c_t}{\eta_t}+33s_{t-1}u_t.
\end{align*}
The last line uses $u_t\ge1$ and $s_{t-1}u_t\ge1$. Also,
\[
\sum_{t=1}^T\frac{c_t}{\eta_t}
=\sum_{t=1}^T\frac{S_c\sqrt{c_t}}{\bar\eps}
=\frac{S_c^2}{\bar\eps}.
\]
Thus the total deterministic raw budget is
\begin{align*}
\sum_{t=1}^T m_t^{\mathrm{raw}}
&=O\left(\sum_t\frac{s_{t-1}(d_t+u_t)}{\eta_t}
       +\sum_ts_{t-1}u_t\right)\\
&=O\left(\frac{S_c^2}{\bar\eps}+\sum_ts_{t-1}u_t\right).
\end{align*}
Using \eqref{eq:stage-dimension} and $\sqrt{a+b}\le\sqrt a+\sqrt b$,
\[
S_c\le T\sqrt{8B}+\sum_t\sqrt{s_{t-1}u_t}.
\]
For $t=1$, $s_0=R\le4R2^{-1}$. For $2\le t\le T$, put $z=R2^{-t}\ge1$. Then $s_{t-1}=2\lceil z\rceil-1\le2z+1\le4z$. Hence $s_{t-1}\le4R2^{-t}$ in every stage. Also
$u_t=1+L+(t+3)\ln2$.
Splitting the square root of the two nonnegative summands in $u_t$ gives
\begin{align*}
\sum_{t=1}^T\sqrt{s_{t-1}u_t}
&\le2\sqrt R\left(\sqrt{1+L}\sum_{t\ge1}2^{-t/2}
 +\sqrt{\ln2}\sum_{t\ge1}\sqrt{t+3}\,2^{-t/2}\right)\\
&=O(\sqrt{R(1+L)}).
\end{align*}
Both infinite sums are universal constants. The first is geometric; consecutive terms in the second have ratio at most $\sqrt{5/8}<1$. Without square roots, the identities $\sum_{t\ge1}2^{-t}=1$ and $\sum_{t\ge1}t2^{-t}=2$ give
\[
\sum_{t=1}^Ts_{t-1}u_t
\le4R\left((1+L)\sum_{t\ge1}2^{-t}
+\ln2\sum_{t\ge1}(t+3)2^{-t}\right)
=4R(1+L+5\ln2).
\]
Squaring the estimate for $S_c$ with $(a+b)^2\le2a^2+2b^2$ yields
$S_c^2=O(BT^2+R(1+L))$.
Substituting this and the bound on $\sum_t s_{t-1}u_t$ into the raw-round estimate gives
\[
\sum_{t=1}^T m_t^{\mathrm{raw}}
=O\left(\frac{BT^2+R(1+L)}{\bar\eps}
       +R(1+L)\right).
\]
Here $T\le\lambda$. Also $\bar\eps\ge\eps/2$: for $\eps\le1/2$ it equals $\eps$, and otherwise it equals $1/2\ge\eps/2$. Finally $\eps<1$ allows the additive term $R(1+L)$ to be absorbed by $R(1+L)/\eps$. These substitutions give
\[
\sum_{t=1}^T m_t^{\mathrm{raw}}
=O\left(\frac{B\lambda^2+R(1+L)}{\eps}\right),
\]
which is the claimed deterministic sample budget.
\end{proof}

\subsection{The Positive-Response Restriction}\label{app:fiberization}
\role{Lemma~\ref{lem:fiberization} turns the observed positive label in Section~\ref{sec:fiber} into a structural bound for its second phase. It uses only the witness definition and Lemma~\ref{lem:dimensions}.}
\begin{proof}\textbf{of Lemma~\ref{lem:fiberization}}
Take any ordinary witness for $\cH_{x=y}$. All its hypotheses agree at $x$, so $x$ cannot be a positive coordinate of that witness. Append $x$ to its sequence with multiplicity zero. Its restriction vectors all acquire the same new entry $y$. Neighbors in each original coordinate remain neighbors, because they still agree in every other coordinate, including the appended one. The result is an anchored witness of the original class with the same value. Taking suprema gives $\BDS(\cH_{x=y})\le A$. Applying the width inequality of Lemma~\ref{lem:dimensions} to this nonempty fiber gives
$\wid(\cH_{x=y})\le\BDS(\cH_{x=y})+1\le A+1$.
\end{proof}

\subsection{The Complete Two-Phase Guarantee}\label{app:structural}
\role{Theorem~\ref{thm:structural} supports the entire two-phase procedure. Lemma~\ref{lem:fiberization} supplies the second-phase class parameters, and Theorem~\ref{thm:cascade} supplies its PAC guarantee.}
\begin{proof}\textbf{of Theorem~\ref{thm:structural}}
\textit{Obtain a true label within the reserved first phase.}
Before the first positive, the target label belongs to the realized set, whose size is at most $R$. Conditional on any such history and the current instance, uniform prediction from that set is positive with probability at least $1/R$. For bookkeeping, predictions in any unused first-phase rounds may be chosen arbitrarily. Let $E_t$ be the event of no positive in the first $t$ rounds, with $E_0$ the whole probability space, and let $\mathcal F_{t-1}$ be the history before round $t$. Averaging over the current instance and then the history gives
\[
\Pr[E_t]
=\E\!\left[\ind{E_{t-1}}
\Pr[b_t=0\mid\mathcal F_{t-1}]\right]
\le(1-1/R)\Pr[E_{t-1}].
\]
This bound uses conditional success probabilities, not independence between the feedback bits. Iterating it for $q=\lceil R\ln(2/\delta)\rceil$ rounds yields
\[
\Pr[E_q]\le(1-1/R)^q\le e^{-q/R}\le\delta/2.
\]

\textit{Apply the cascade to the surviving class.}
On a first positive $(x,y)$, the target lies in $\cH_{x=y}$. As specified in Section~\ref{sec:fiber}, store this pair and start phase two only after the fixed reserved time $q$, ignoring intervening observations. The subsequent observations are therefore i.i.d. draws from $\cD$, independent of the first-phase history. For every possible stored $(x,y)$, Lemma~\ref{lem:fiberization} bounds the surviving class's parameters by $A$ and $A+1$. Its logarithmic width factor is at most $\lambda_A$, since $r\mapsto\max\{1,\log_2 r\}$ is nondecreasing for $r\ge1$. Running Theorem~\ref{thm:cascade} on it with confidence $\delta/2$ thus uses
\[
O\left(\frac{A\lambda_A^2+(A+1)(1+\ln(2/\delta))}{\eps}\right)
\]
additional rounds. The bound is uniform over the stored pair, so it also holds conditional on the random first phase. Allocate this uniform deterministic upper budget; if the actual fiber needs fewer rounds, unused rounds may be ignored.

\textit{Combine the two conditional guarantees.}
Let $\mathcal S=E_q^c$ be first-phase success and let $\mathcal F_q$ be the entire first-phase history. Conditional on every history in $\mathcal S$, the second phase fails with probability at most $\delta/2$. Hence
\[
\Pr[\text{final risk}>\eps]
\le\Pr[\mathcal S^c]
 +\E\!\left[\ind{\mathcal S}
 \Pr[\text{second phase fails}\mid\mathcal F_q]\right]
\le\delta/2+\delta/2=\delta.
\]
The first-phase ceiling is absorbed in $R\ln(2/\delta)$, since $R\ge1$ and $\ln(2/\delta)\ge\ln2$. Adding the budgets proves the theorem. At $\delta=1/4$, $\lambda_A=O(\log(A+2))$, giving the displayed constant-confidence specialization.
\end{proof}

\Needspace{12\baselineskip}
\section{Additional Comparisons and Boundary Cases}\label{app:supplement}
The main proofs are complete before this appendix. The results here clarify parameter choices, preserve the larger-confidence endpoint, and distinguish an algorithmic limitation from a lower bound on all learners.

\subsection{Alternative Anchoring Rules}\label{app:alternative-anchor}
\begin{proposition}[Zeroing the smallest multiplicity]\label{prop:alternative-anchor}
Let $T_+$ be the supremum of $\sum_iN_i-\min_iN_i$ over ordinary witnesses, whose multiplicities are all positive. Let $T_0$ use the same expression over witnesses with $d\ge2$ and nonnegative multiplicities, without requiring a zero entry. Empty suprema are zero. Then $T_0=A$ and $T_+\le A$. The gap $A-T_+$ can be arbitrarily large.
\end{proposition}
\role{This explains why Section~\ref{sec:model} uses zero-multiplicity anchors rather than merely dropping one term from an ordinary witness.}
\begin{proof}
In an anchored witness the minimum multiplicity is zero, so its $T_0$ value equals its anchored value. Thus $T_0\ge A$. Conversely, take any witness considered for $T_0$ and replace a minimum multiplicity by zero. This weakens a neighbor requirement and produces an anchored witness with value $\sum_iN_i-\min_iN_i$. Hence $T_0\le A$. The same operation works on an ordinary witness with at least two coordinates, proving $T_+\le A$; a one-coordinate witness contributes zero to $T_+$.

For the rare-label class $\cG_K$, all variation is at the single point $b$. Lemma~\ref{lem:coordinates} precludes repeating that point as an active coordinate. Every ordinary witness therefore has only that positive coordinate and contributes zero to $T_+$. Corollary~\ref{cor:rare} gives $A=K-1$, which grows arbitrarily large.
\end{proof}

Uniformity on a conditional fiber does not require global label symmetry. For example, the restriction family
\[
\{(0,1,1),(0,2,1),(0,1,2),(0,2,2),(0,3,2),(0,2,3),(0,3,3)\}
\]
is anchored in its first coordinate and has at least one neighbor in each of the other coordinates at every vector. Swapping labels $2$ and $3$ in the second coordinate sends $(0,2,1)$ outside the family. Thus that swap cannot justify a symmetry reduction for this witness. Lemma~\ref{lem:equality} still applies, because conditioning on all but one coordinate gives a uniform fiber regardless of such swaps.

\subsection{Several Scales Can Still Be Localized}\label{app:scales}
\begin{proposition}[Multi-scale PAB classes]\label{prop:scales}
Fix integers $D\ge1$ and $J\ge1$. Take disjoint blocks $X_j$ of size $\lceil D/2^j\rceil$ for $j=0,\ldots,J-1$, with $2^j$ non-base labels at every point of $X_j$, a separate common anchor, and no private hypotheses. Then
\[
A=B=\max_{0\le j<J}2^j\lceil D/2^j\rceil,
\qquad R=1+2^{J-1},
\]
and the standard-range sample complexity is $\Theta((B+L)/\eps)$.
If $2^{J-1}$ divides $D$, then $B=D$ and
$2^j\DS_{2^j}(\cH)=D$ for every $j=0,\ldots,J-1$.
\end{proposition}
\role{This backs the multi-scale observation after Theorem~\ref{thm:pab}. It shows why a large sum over label scales need not force a large optimal learning rate.}
\begin{proof}
This is a PAB representation with block widths
$M_j=2^j\lceil D/2^j\rceil$.
The separate anchor ensures that every positive block has an external anchor, including when $J=1$. Also $\max_j M_j\ge2^{J-1}=R-1$. Theorem~\ref{thm:pab} therefore gives $A=B=\max_jM_j$ and rate $\Theta(R-1+(B+L)/\eps)$. Since $R-1\le B$ and $\eps<1$, the quantity inside this last order bound lies between $(B+L)/\eps$ and $2(B+L)/\eps$. This proves the displayed matching rate with universal constants.

If $2^{J-1}$ divides $D$, each $M_j=D$. The full product on block $X_j$ witnesses
$\DS_{2^j}\ge|X_j|=D/2^j$.
The reverse inequality follows from \eqref{eq:list-dimension} with $B=D$. Hence equality holds at every stated scale, and the unweighted sum of the $J$ scale values is exactly $JD$.
\end{proof}

There is a separate budget limitation for a stagewise design whose stage costs are bounded below by $cB/\eta_t$ for a constant $c>0$, and whose positive error allowances obey $\sum_{t=1}^T\eta_t\le\eps$. Cauchy--Schwarz applied to the sequences $(\sqrt{\eta_t})_{t=1}^T$ and $(1/\sqrt{\eta_t})_{t=1}^T$ gives
\[
T^2=\left(\sum_{t=1}^T\sqrt{\eta_t}\frac1{\sqrt{\eta_t}}\right)^2
\le\left(\sum_{t=1}^T\eta_t\right)
   \left(\sum_{t=1}^T\frac1{\eta_t}\right).
\]
Dividing by the positive sum of the error allowances and multiplying by $cB$ yields
\[
\sum_{t=1}^T\frac{cB}{\eta_t}
\ge\frac{cBT^2}{\sum_{t=1}^T\eta_t}
\ge\frac{cBT^2}{\eps}.
\]
This inequality applies to the stage-cost model specified above. In contrast, the PAB learner attains the rate in Proposition~\ref{prop:scales} through localization, even when many list-dimension scales are simultaneously large.

\Needspace{9\baselineskip}
\subsection{Finite Dimension Does Not Supply a Mistake Budget}
\begin{proposition}[A threshold example]\label{prop:thresholds}
Let $\cY=\{1,2\}$ and $\cH=\{h_\theta:0<\theta<1\}$ on $(0,1)$, where
$h_\theta(x)=1$ for $x<\theta$ and $h_\theta(x)=2$ otherwise. Then $A=B=1$ and $R=2$. For every learner, every fixed seed, and every integer $m\ge1$, there are a target in $\cH$ and a realizable length-$m$ instance sequence producing $m$ negative responses. Hence no finite pathwise bound holds uniformly over all horizons and realizable sequences.
\end{proposition}
\role{This delimits the use of Lemma~\ref{lem:monotone}: the general upper-bound module needs i.i.d. list learning rather than an adversarial mistake bound obtained from $B$.}
\begin{proof}
Both labels occur at every domain point, so $R=2$ and $B\ge1$. In a binary ordinary witness every multiplicity is one. If it had two or more active coordinates, projection onto two of them would preserve both neighbor directions: choose a preimage of each projected vector and project its neighbor in the desired coordinate. By Lemma~\ref{lem:binary-cube}, the projection would contain all four binary patterns. But at ordered points $x_1<x_2$, a threshold cannot have labels $(2,1)$, because $x_1\ge\theta$ implies $x_2\ge\theta$. Thus no such witness exists and $B=1$. Two different thresholds agree below both thresholds and disagree between them. Lemma~\ref{lem:dimensions} gives $A>0$, while $A\le B=1$ and integrality imply $A=1$.

We now disprove a finite pathwise mistake bound. Fix a learner seed and any positive integer $m$. Start with the possible parameter interval $I_0=(0,1)$. At round $t$, write the current interval as $I_{t-1}=(a_t,b_t)$ and present the instance $x_t=(a_t+b_t)/2$. If the learner predicts label $1$, give negative feedback and set $I_t=(a_t,x_t)$. Every threshold in this interval labels $x_t$ by $2$. If it predicts label $2$, again give negative feedback and set $I_t=(x_t,b_t)$. Every threshold in this interval labels $x_t$ by $1$.

Each retained interval is nonempty and contained in its predecessor. Inductively, every parameter in $I_t$ realizes all responses through round $t$. Choose any $\theta^\star\in I_m$. Replaying the constructed instance sequence with target $h_{\theta^\star}$ and the fixed seed produces exactly the same predictions and negative responses. Thus there is a fixed target and a length-$m$ realizable sequence with $m$ negatives. Taking $m$ larger than any proposed finite bound refutes that bound. No common target for an infinite sequence is needed for this conclusion. For a randomized algorithm, fixing a seed suffices because the property in question is required on every seed and every realizable sequence. This concerns that pathwise adversarial property, not a lower bound for i.i.d. PAC learning.
\end{proof}

\Needspace{9\baselineskip}
\subsection{The Large-Failure-Probability Endpoint}\label{app:large-delta}
\begin{proposition}[Zero-sample threshold of the multiplexer]\label{prop:large-delta}
For $n\ge16$ divisible by four, $0<\eps\le1/16$, and $0<\delta<1$,
\[
\mB_{\AM}(\eps,\delta)
\ge\left\lceil\frac n{4\eps}
\left(1-\delta-\frac1{n2^n}\right)_+\right\rceil,
\]
where $(z)_+=\max\{z,0\}$. Moreover,
$\mB_{\AM}(\eps,\delta)=0$ if and only if $\delta\ge1-1/(n2^n)$.
\end{proposition}
\role{This is the endpoint supplement referred to after Theorem~\ref{thm:profile}; it is not used to establish the standard-range profile.}
\begin{proof}
Use the hard distribution of Appendix~\ref{app:am}, and for any two predictors define their disagreement distance by
$\rho(f,g)=\Pr_{x\sim\cD}[f(x)\ne g(x)]$.
Different-group targets disagree on every informative point and have distance $4\eps$. For same-group targets with distinct parameters $\theta,\theta'$, put $u=\theta-\theta'\ne0$. Choose $v\in\F_2^n$ with $\langle u,v\rangle=1$. The map $x\mapsto x+v$ pairs vectors whose values under $\langle u,\cdot\rangle$ are opposite. Exactly $2^{n-1}$ vectors in $\F_2^n$ therefore have value one. All of them are nonzero, so the two targets differ on exactly $2^{n-1}$ points of $V_n$. Their distance is
\[
\rho(h_{g,\theta},h_{g,\theta'})
=4\eps\,\frac{2^{n-1}}{2^n-1}>2\eps.
\]
If a deterministic output $\hat h$ had risk at most $\eps$ for two distinct targets $h,h'$, the pointwise inequality
$\ind{h(x)\ne h'(x)}\le\ind{h(x)\ne\hat h(x)}+\ind{\hat h(x)\ne h'(x)}$
would give $\rho(h,h')\le2\eps$, a contradiction. Thus a single output can succeed for at most one target. Fixing the output randomness makes this statement applicable to randomized learners as well.

For a fixed seed and instance sequence, call a group touched if the canonical negative execution predicts it on at least one informative round. Let $\tau$ count touched groups and let $N$ count informative observations, so $\tau\le N$. Every target in an untouched group produces the canonical record by Lemma~\ref{lem:canonical}. These targets have total prior mass $(n-\tau)/n$, and the canonical output succeeds for at most one target, of prior mass $1/(n2^n)$. Therefore the conditional prior-average failure probability is at least $(n-\tau)/n-1/(n2^n)$. The untouched set is determined by the seed and instance sequence, not by the random target. We can consequently average this bound against the unchanged uniform target prior. Since $N\sim\operatorname{Bin}(m,4\eps)$ and $\E\tau\le\E N=4\eps m$, the resulting failure probability is at least
\[
1-\frac{\E\tau}{n}-\frac1{n2^n}
\ge1-\frac{4\eps m}{n}-\frac1{n2^n}.
\]
If each target's failure probability is at most $\delta$, the same is true of this average. Hence
\[
\delta\ge1-\frac{4\eps m}{n}-\frac1{n2^n},
\qquad
m\ge\frac n{4\eps}\left(1-\delta-\frac1{n2^n}\right).
\]
A sample budget is also nonnegative. Taking the maximum with zero, then taking the ceiling because $m$ is an integer, proves the stated lower bound.

For $m=0$, the argument requires $\delta\ge1-1/(n2^n)$. Conversely, the class has exactly $n2^n$ distinct hypotheses: distinct groups differ in their first component, and distinct parameters are distinguished by a nonzero linear functional. A zero-sample learner that chooses one uniformly has probability $1/(n2^n)$ of being exactly correct for each target, under every instance distribution. This achieves the threshold.
\end{proof}

For completeness, this endpoint also extends a width-times-confidence lower bound beyond $\delta=1/4$. For $\delta\ge1/4$,
\[
\ln(1/\delta)=\int_\delta^1\frac1u\,du\le4(1-\delta).
\]
If also $\delta\le1-2^{-19}$, then $1/(n2^n)\le2^{-20}\le(1-\delta)/2$. The expression inside the positive part in Proposition~\ref{prop:large-delta} is therefore at least $(1-\delta)/2$. Substituting this and the logarithmic inequality into that proposition gives
$\mB_{\AM}\ge nL/(32\eps)$.
More explicitly, fix $1/4\le\bar\delta<1-2^{-20}$ and define
$c_{\bar\delta}=1-2^{-20}/(1-\bar\delta)>0$.
For $1/4\le\delta\le\bar\delta$ and admissible $n$,
\[
1-\delta-\frac1{n2^n}\ge c_{\bar\delta}(1-\delta),
\qquad
\mB_{\AM}\ge\frac{c_{\bar\delta}n(1-\delta)}{4\eps}
\ge\frac{c_{\bar\delta}nL}{16\eps}.
\]
The last inequality again uses $L\le4(1-\delta)$. At $\bar\delta=1-2^{-20}$ the $n=16$ class is zero-sample learnable, so no positive uniform lower bound of this form can persist there.

\subsection{What the Sharper List Input Changes}\label{app:list-compatibility}
This comparison explains the collection cost discussed in Section~\ref{sec:cascade}. Proposition~\ref{prop:list-blackbox} already uses the sharper input of \citet{Pabbaraju26}, together with suffix-risk conversion and strict-majority aggregation. Unlike the list-PAC bound in \citet[Lemma~4.6]{HMMSv3}, its dimension term has no factor $\log_2R$. The proof in Appendix~\ref{app:upper} verifies the list-size interface and the constant used by our algorithm.

At stage $t$, retain the notation of Section~\ref{sec:cascade}: the preceding list width is $s_{t-1}$, the dimension is $d_t$, the conditional error allowance is $\eta_t$, and the confidence overhead is $u_t=1+\ln(1/\alpha_t)$. The number of labeled inputs is $\lceil100(d_t+u_t)/\eta_t\rceil$. Lemma~\ref{lem:harvest} translates this into the sufficient raw-round budget
\[
O\left(\frac{s_{t-1}d_t}{\eta_t}
+\frac{s_{t-1}u_t}{\eta_t}+s_{t-1}u_t\right).
\]
Thus the same collection step multiplies both the dimension and confidence parts of the labeled-sample requirement by the list width. The complete allocation and summation are proved in Appendix~\ref{app:cascade}, yielding Theorem~\ref{thm:cascade}. The resulting dimension term is $B\lambda^2/\eps$, whereas the confidence contribution remains $R(1+L)/\eps$.

This calculation concerns the sufficient budget of our implementation. On the multiplexer, Theorem~\ref{thm:interval} independently gives an $\Omega(RL/\eps)$ worst-case lower bound on every learner's total sample budget in its stated confidence interval. The two results separate the benefit of a better full-information input from the cost of obtaining its labels through bandit feedback.

\Needspace{8\baselineskip}
\subsection{Version-Specific Comparison with HMMS}\label{app:versions}
The correction in Section~\ref{sec:baseline} concerns the unspecified side condition in \citet[Theorem~4.2]{HMMSv2}. Our pair-anchored example satisfies agreement somewhere and disagreement elsewhere, but does not give a common anchor for every member of the class.

In \citet[Section~3 and Theorem~4.3]{HMMSv3}, nontriviality requires an ordinary witness attaining $B$ on at least two instances. For such a witness on $d\ge2$ coordinates, a smallest multiplicity is at most $B/d\le B/2$. Zeroing that multiplicity makes an anchored witness of value at least $B/2$, so $A\ge B/2$. Our all-class baseline therefore recovers a lower bound of order $(B+L)/\eps$ on that restricted family.

The additive upper-bound question in Section~5 of v3 concerns all classes of finite bandit DS dimension. Corollary~\ref{cor:separation} answers it negatively for constants uniform over classes. The proof of Lemma~\ref{lem:am-dimensions} shows that for $n\ge2$ every ordinary multiplexer witness on at least two coordinates has value at most $n<2n-1=B$. Thus the multiplexer lies within the scope of that question while failing the nontriviality requirement of the revised lower-bound theorem.

The present contribution is the all-class anchored formulation, its matching families, the confidence profile, and the refined allocation in the general upper bound. Proposition~\ref{prop:list-blackbox} derives the sharper full-information list input from \citet{Pabbaraju26} and \citet{AdenAliEtAl23}, retaining the strict-majority aggregation used in v3. Our bandit reductions use this input with the allocation analyzed in Appendix~\ref{app:cascade}.

\end{document}